\documentclass[10pt,english]{article}
 
\usepackage[T1]{fontenc}
\usepackage[utf8]{inputenc}
\usepackage{babel}
\usepackage{subfigure}  
\usepackage{booktabs} 
\usepackage{mathrsfs}

\usepackage{hyperref}
\hypersetup{
 colorlinks,linkcolor=red,anchorcolor=blue,citecolor=blue}

\usepackage{amsthm}
\usepackage{amsmath}
\usepackage{amssymb}
\usepackage{thmtools}
\usepackage{thm-restate}

\usepackage{hyperref}

\usepackage{color}

\usepackage{fancyhdr}
\usepackage{bm,amsmath,amssymb,graphicx,graphics,subfigure,epsfig,enumerate}

\usepackage{cleveref}

\usepackage{eso-pic} 
\usepackage{forloop}
\usepackage[normalem]{ulem}

\usepackage{algorithm}
\usepackage{algorithmic}

\usepackage[margin=1in]{geometry}

\DeclareMathOperator*{\argmin}{arg\,min}

\newtheorem{assumption}{\textbf{Assumption}}\newtheorem{definition}{\textbf{Definition}}\newtheorem{lemma}{\textbf{Lemma}}\newtheorem{theorem}{\textbf{Theorem}}\newtheorem{remark}{\textbf{Remark}}

\newcommand{\cM}{\mathcal{M}}

\newcommand{\tx}{\tilde{x}}

\newcommand{\norm}[1]{\left\lVert#1\right\rVert}
\newcommand{\tth}[1]{#1^{\text{th}}}
\newcommand{\ex}[1]{\mathbb{E}\left[#1\right]}
\newcommand{\cond}{\bigg |}
\newcommand{\paren}[1]{\left(#1\right)}
\newcommand{\innprod}[1]{\left\langle#1\right\rangle}

\newcommand{\mbR}{\mathbb{R}}
\newcommand{\subx}{y}

\makeatother

\usepackage{authblk}

\title{Convergence of Distributed Stochastic Variance Reduced Methods without Sampling Extra Data}

\author{Shicong Cen$^\dag$ \quad Huishuai Zhang$^\ddag$\quad 
Yuejie Chi$^\dag$ \quad
Wei Chen$^\ddag$ \quad
Tie-Yan Liu$^\ddag$}

\affil{$^\dag$Carnegie Mellon University \quad  $^\ddag$Microsoft Research Asia \\ \vspace{0.02in}
\texttt{\{shicongc,yuejiec\}@andrew.cmu.edu} \\
\texttt{\{huzhang, wche, tyliu\}@microsoft.com}}
 
\date{May 2019; Revised Jun. 2020}
\begin{document}

\maketitle

\begin{abstract}

Stochastic variance reduced methods have gained a lot of interest recently for empirical risk minimization due to its appealing run time complexity. When the data size is large and disjointly stored on different machines, it becomes imperative to distribute the implementation of such variance reduced methods. In this paper, we consider a general framework that directly distributes popular stochastic variance reduced methods in the master/slave model, by assigning outer loops to the parameter server, and inner loops to worker machines. 
This framework is natural and friendly to implement, but its theoretical convergence is not well understood. We obtain a comprehensive understanding of algorithmic convergence with respect to data homogeneity by measuring the smoothness of the discrepancy between the local and global loss functions.
We establish the linear convergence of distributed versions of a family of stochastic variance reduced algorithms, including those using accelerated and recursive gradient updates, for minimizing strongly convex losses. Our theory captures how the convergence of distributed algorithms behaves as the number of machines and the size of local data vary. Furthermore, we show that when the data are less balanced, regularization can be used to ensure convergence at a slower rate. We also demonstrate that our analysis can be further extended to handle nonconvex loss functions.

\end{abstract}

	\section{Introduction}

 	Empirical risk minimization arises frequently in machine learning and signal processing, where the objective function is the average of losses computed at different data points. Due to the increasing size of data, distributed computing architectures, which assign the learning task over multiple computing nodes, are in great need to meet the scalability requirement in terms of both computation power and storage space. In addition, distributed frameworks are suitable for problems where there are privacy concerns to transmit and store all the data in a central location, a scenario related to the nascent field of {\em federated learning} \cite{konevcny2015federated}. It is, therefore, necessary to develop distributed optimization frameworks that are tailored to solving large-scale empirical risk minimization problems with desirable communication-computation trade-offs, where the data are stored disjointly over different machines.

 One prevalent model of distributed systems is the so-called master/slave model, where there is a central parameter server to coordinate the computation and information exchange across different worker machines. Due to the low per-iteration cost, a popular solution is distributed stochastic gradient descent (SGD) \cite{recht2011hogwild}, where the parameter server aggregates gradients from each worker and does mini-batch gradient updates. However, distributed SGD is not communication-efficient and requires lots of communication rounds to converge, which partially diminishes the benefit of distribution. On the other hand, recent breakthroughs in developing stochastic variance reduced methods have made it possible to achieve fast convergence and small per-iteration cost at the same time, such as the notable SVRG \cite{Johnson2013} algorithm. Yet, distributed schemes of such variance reduced methods that are both practical and theoretically sound are much less developed.  
 
This paper focuses on a general framework of distributed stochastic variance reduced methods, to be presented in Alg.~\ref{alg:general}, which is natural and friendly to implement in a master/slave model. On a high level, SVRG-type algorithms \cite{Johnson2013} contain inner loops for parameter updates via variance-reduced SGD, and outer loops for global gradient and parameter updates. Our general framework assigns outer loops to the parameter server, and inner loops to worker machines. The parameter server collects gradients from worker machines and then distributes the global gradient to each machine. Each worker machine then runs the inner loop independently in parallel using variance reduction techniques, which might be different when distributing different algorithms, and returns the updates to the parameter server at the end.  Per iteration,  two communication rounds are required: one communication round is used to average the parameter estimates, and the other is used to average the gradients, which is the same as distributed synchronous SGD. However, the premise is that by performing more efficient local computation using stochastic variance reduced methods, the algorithm converges in fewer iterations and is therefore more communication-efficient.

Due to the simplicity of this framework, similar methods have been implemented in several works \cite{konevcny2015federated, de2016efficient, reddi2016aide}, and have achieved  great empirical success. Surprisingly, a complete theoretical understanding of its convergence behavior is still missing at large. Moreover, distributed variants using accelerated variance reduction methods are not developed. The main analysis difficulty is that the variance-reduced gradient of each worker is no longer an unbiased gradient estimator when sampling from re-used local data.

To ease this difficulty, several variants of distributed SVRG, e.g. \cite{lee2017distributed,shamir2016without,wang2017memory} have been proposed with performance guarantees, which try to bypass the biased gradient estimation issue by emulating the process of i.i.d. sampling from the global data using some complicated random data re-allocation protocol, which requires sampling extra data with or without replacement. These procedures lead to unnecessary data waste and potential privacy leakage, and can be cumbersome and difficult to implement in practice. 
 
Consequently, a natural question arises: {\em can we provide a mathematical analysis to the convergence of the natural framework of distributed stochastic variance reduced methods, under some simple and intuitive metric?}

 \subsection{Contributions of This Paper}

This paper provides a convergence analysis of a family of {\em naturally} distributed stochastic variance reduced methods under the framework described in Alg.~\ref{alg:general}, for both convex and nonconvex loss functions. By using different variance reduction schemes at the worker machines, we study distributed variants of three representative algorithms in this paper: SVRG \cite{Johnson2013}, SARAH employing recursive gradient updates \cite{Nguyen2017,nguyen2019optimal}, and MiG employing accelerated gradient updates \cite{zhou2018simple}. Our methodology can be extended to study other variants in a similar fashion. The contributions of this paper are summarized below.
 	\begin{itemize}
 	\item We suggest a simple and intuitive metric called {\em distributed smoothness} to gauge data balancedness among workers, defined as the smoothness of the difference $f_k -f$ between the local loss function $f_k$ and the global loss function $f$, which is the average of the local loss functions. The metric is deterministic, easy-to-compute, applies for arbitrary dataset splitting, and is shown to play a critical role in the convergence analysis.  
		
\item	We establish the linear convergence of distributed D-SVRG, D-SARAH, and D-MiG under strongly convex losses, as long as the distributed smoothness parameter is smaller than a constant fraction of the strong convexity parameter $\sigma$, e.g. $\sigma/4$, where the fraction might change for different algorithms. Our bounds capture the phenomenon that the convergence rate improves as the local loss functions become more similar to the global loss function, by reducing the distributed smoothness parameter. Furthermore, the run time complexity exhibits the so-called ``linear speed-up'' property in distributed computing, where the complexity depends on the local data size, instead of the global data size, which typically implies an improvement by a factor of $n$, where $n$ is the number of machines.

        \item When the local data are highly unbalanced, the distributed smoothness parameter becomes large, which implies that the algorithm might diverge. We suggest regularization as an effective way to handle this situation, and show that by adding larger regularization to machines that are less distributed smooth, one can still ensure linear convergence in a regularized version of D-SVRG, called D-RSVRG, though at a slower rate of convergence. 
        
        \item More generally, the notion of distributed smoothness can also be used to establish the convergence under nonconvex losses. We demonstrate this through the convergence analysis of D-SARAH in the nonconvex setting.
        
 	\end{itemize}

    \begin{table*}[t]
    	    \caption{Communication rounds and runtime of the proposed and existing algorithms for strongly convex losses (ignoring logarithmic factors in $\kappa$) to reach $\epsilon$-accuracy. Algorithms with an asterisk are proposed/analyzed in this paper. Here, $N$ is the total number of input data points, $n$ is the number of worker machines, and $\kappa$ is the condition number of the global loss function $f$.}   
   \centering
	\begin{tabular}{ |l | c |  c| c|}
		\hline
		Algorithm &  Communication Rounds & Runtime & Assumptions \\ \hline
		\hline
		DSVRG \cite{lee2017distributed} &  $ (1+\kappa/(N/n))\log(1/\epsilon)$  & $(N/n+\kappa)\log (1/\epsilon)$&  extra data \\ \hline		
		DASVRG \cite{lee2017distributed} & $ (1+\sqrt{\kappa/(N/n)})\log(1/\epsilon) $ & $(N/n+\sqrt{\kappa N/n})\log (1/\epsilon)$ & extra data\\\hline
		Dist. AGD   & $ \sqrt{\kappa}\log (1/\epsilon)$ & $(N/n)\sqrt{\kappa}\log (1/\epsilon)$ & none \\ \hline
		ADMM &$\kappa \log(1/\epsilon) $ & $(N/n) \kappa \log(1/\epsilon) $ & none \\ \hline
		SCOPE \cite{zhao2017scope} & $\kappa \log(1/\epsilon)$ & $(N/n + \kappa\log\kappa)\kappa \log(1/\epsilon)$  &  uniform regularization \\ \hline
		pSCOPE \cite{zhao2018proximal} & $\log(1/\epsilon)$  & $(N/n + \kappa)\log(1/\epsilon)$ & good partition \\ \hline
		D-SVRG*  &  $\log(1/\epsilon)$  & $(N/n + \kappa)\log(1/\epsilon)$ & distributed restricted smoothness \\ \hline
 		D-SARAH*  &  $\log(1/\epsilon)$  & $(N/n + \kappa)\log(1/\epsilon)$ & distributed restricted smoothness \\ \hline
		D-MiG*  &  $ (1+\sqrt{\kappa/(N/n)})\log(1/\epsilon)$   & $(N/n+\sqrt{\kappa N/n})\log(1/\epsilon)$ & distributed smoothness \\  \hline
		D-RSVRG*   &  $ \kappa\log(1/\epsilon)$  & $ (N/n)  \kappa \log(1/\epsilon)$ & large regularization \\ \hline
		D-RSVRG*   &  $ \log(1/\epsilon)$  & $(N/n + \kappa )\log(1/\epsilon)$ & small regularization \\ \hline
	\end{tabular}
	\label{tab:dist-opt-algs}
\end{table*}

 \subsection{Related Work}
 
Distributed optimization is a classic topic \cite{boyd2011distributed,bertsekas1989parallel}, yet recent  trends in data-intensive applications are calling for new developments with a focus on communication and computation efficiency. Examples of {deterministic} optimization methods include DANE \cite{Shamir2013}, AIDE \cite{reddi2016aide}, DiSCo \cite{Zhang2015}, GIANT \cite{wang2018giant}, CoCoA \cite{smith2018cocoa}, CEASE \cite{fan2019communication}, one-shot averaging \cite{zinkevich2010parallelized, zhang2012communication}, etc.

Many stochastic variance reduced methods have been proposed recently, for example, SAG \cite{Schmidt2015}, SAGA \cite{defazio2014saga}, SVRG \cite{Johnson2013}, SDCA \cite{shalev2013stochastic}, MiG \cite{zhou2018simple}, Katyusha \cite{allen2017katyusha}, Catalyst \cite{lin2015universal}, SCOPE \cite{zhao2017scope,zhao2018proximal}, SARAH \cite{Nguyen2017}, SPIDER \cite{fang2018spider}, SpiderBoost \cite{wang2018spiderboost}, to name a few. Several previous works have studied distributed variants of SVRG. For example, the D-SVRG algorithm has been empirically studied before in \cite{reddi2016aide,konevcny2015federated} without a theoretical convergence analysis. The pSCOPE algorithm \cite{zhao2018proximal} is also a variant of distributed SVRG, and its convergence is studied under an assumption called good data partition in \cite{zhao2018proximal}, which is hard to interpret and verify in practice. The SCOPE algorithm \cite{zhao2017scope} is similar to the regularized variant D-RSVRG of D-SVRG under large regularization, however our analysis is much more refined by allowing different regularizations to different local workers with respect to the distributed smoothness of local data, and gracefully degenerates to the unregularized case when the distributed smoothness is benign. The general framework of distributed variance-reduced methods covering SARAH and MiG and various loss settings in this paper has not been studied before. For conciseness, Table~\ref{tab:dist-opt-algs} summarizes the communication and runtime complexities of the most relevant algorithms to the current paper.\footnote{Since in a master/slave model, all worker machines perform computations in parallel, the runtime is dictated by the worst-case computation complexity per worker.}

There are also a lot of recent efforts on reducing the communication cost of distributed GD/SGD by gradient quantization \cite{alistarh2017qsgd, bernstein2018signsgd, seide20141, wen2017terngrad}, gradient compression and sparsification \cite{alistarh2018convergence, lian2017can, lin2018deep, wangni2018gradient, tang2018d}. In comparison, we communicate the exact gradient, and it is an interesting future direction to combine gradient compression schemes in distributed variance reduced stochastic gradient methods. Another line of works \cite{yuan2017variancereduced,mokhtari2015dsa,li2020communication} seek to adopt variance reduction techniques in the decentralized setting, where each worker is only able to communicate with its neighbors over a network. In contrast, our work focuses on the master/slave model where each worker communicates with a parameter server.

\subsection{Paper Organization} The rest of this paper is organized as follows. Section~\ref{sec:problem_setup} presents the problem setup and a general framework of distributed stochastic optimization with variance-reduced local updates. Section~\ref{sec:main_results} presents the convergence guarantees of D-SVRG, D-SARAH and D-MiG under appropriate distributed smoothness assumptions. Section~\ref{sec:regSVRG} introduces regularization to D-SVRG to handle unbalanced data when distributed smoothness does not hold. Section~\ref{sec:extensions} presents extensions to nonconvex losses for D-SARAH. Section~\ref{sec:outline} provides an outline to the analysis and the rest of the proofs are deferred to the appendix. Section~\ref{sec:numerical} presents numerical experiments to corroborate the theoretical findings. Finally, we conclude in Section~\ref{sec:conclusions}.

\section{Problem Setup} \label{sec:problem_setup}

 	Suppose we have a data set $\cM=\{z_1, \cdots, z_N\}$, where $z_j \in\mbR^p$ is the $j$th data point for $j=1,...,N$, and $N$ is the total number of data points. In particular, we do not make any assumptions on their statistical distribution. Consider the following empirical risk minimization problem 
	\begin{equation} \label{eq:central}
	\min_{x\in\mbR^d} f(x) := \frac{1}{N}\sum_{z\in \cM} \ell(x;z),
	\end{equation}
	where $x\in\mbR^d$ is the parameter to be optimized and $\ell:\mbR^d \times \mbR^p \mapsto \mbR$ is the sample loss function. For brevity, we use $\ell_z(x)$ to denote $\ell(x;z)$ throughout the paper.
	
	In a distributed setting, where the data are distributed to $n$ machines or workers, we define a partition of the data set $\cM$ as 
	$\cM = \bigcup_{k=1}^n \cM_k$, where $\cM_i \bigcap \cM_k = \varnothing$, $\forall i \neq k$.
The $\tth{k}$ worker, correspondingly, is in possession of the data subset $\cM_k$, $1\leq k\leq n$. We assume there is a parameter server (PS) that coordinates the parameter sharing among the workers. The sizes of data held by each worker machine is $N_k = | \cM_k|$. When the data is split equally, we have $N_k = N/n$. The original problem \eqref{eq:central} can be rewritten as minimizing the following objective function:
	\begin{equation}
	f(x):= \frac{1}{n}\sum_{k=1}^{n}f_k(x) ,\label{eq:objectve}
	\end{equation} 
	where $f_k(x)=\frac{1}{(N/n)}\sum_{z\in \cM_k} \ell_z(x)$ is the local loss function at the $\tth{k}$ worker machine.\footnote{It is straightforward to state our results under unequal data splitting with proper rescaling.} 
	
		\begin{algorithm}[ht]
	\caption{A general distributed framework for empirical risk minimization}
	\label{alg:general}
	\begin{algorithmic}[1]
		\STATE {\textbf{Input:} initial point $\tx^0$.}
		\STATE {\textbf{Initialization}: Compute $\nabla f(\tx^0)$ and distribute it to all machines. }
		\FOR {$t = 0,1, 2, \cdots$}
        \FOR{\textbf{workers} $1\leq k \leq n$ in parallel}
        \STATE{Compute
        $y_k^{t+} = \texttt{LocalUpdate}(f_k, \tx^{t}, \nabla f(\tx^t))$;}
        \STATE{Send $y_k^{t+}$ to \textbf{PS};}  

        \ENDFOR

		\STATE{\textbf{PS}: randomly select $\tx^{t+1}$ from all $y_k^{t+}$ and push $\tx^{t+1}$ to all workers;}
		\FOR{\textbf{workers} $1\leq k \leq n$ in parallel}
		\STATE{Compute $\nabla f_k(\tx^{t+1})$ and send it to \textbf{PS}; }
        \ENDFOR
        \STATE {\textbf{PS}: average $\nabla f(\tx^{t+1}) = \frac{1}{n}\sum_{k=1}^n \nabla f_k(\tx^{t+1})$ and push $\nabla f(\tx^{t+1})$ to all workers.}

		\ENDFOR	
		\STATE{\textbf{return} 	$\tx^{t+1}$}	
	\end{algorithmic}
\end{algorithm}	

\subsection{A General Distributed Framework}	 
 Alg.~\ref{alg:general} presents a general framework for distributed stochastic variance reduced methods, which assigns the outer loops of an SVRG-type algorithm \cite{Johnson2013,Nguyen2017,zhou2018simple} to PS and the inner loops to local workers. By using different variance reduction schemes at the worker machines (i.e. \texttt{LocalUpdate}), we obtain distributed variants of different algorithms. On a high level, the framework alternates between local computation by individual workers in parallel (i.e. Line 5), and global information sharing coordinated by the PS (i.e. Lines 8-12). 
 \begin{itemize}
 \item The local worker conducts local computation $\texttt{LocalUpdate}$ based on the current estimate $\tilde{x}^t$, the global gradient $\nabla f(\tilde{x}^t)$, and its local data $f_k(\cdot)$; in this paper, we are primarily interested in local updates using stochastic variance-reduced gradients. A little additional information about the previous update is needed when employing acceleration, which will be specified in Alg.~\ref{alg:dismig_main}.
 \item The PS randomly selects from the local estimates $y_k^{t+}$ from all workers in each round, which is set as the global estimate $\tilde{x}^{t+1}$, then computes the global gradient $\nabla f(\tilde{x}^{t+1})$ by pulling the local gradient  $\nabla f_k(\tilde{x}^{t+1})$. In total, each iteration requires two rounds of communications. %
 \end{itemize}

\begin{remark}\label{remark:alternating_PS_rule}
A careful reader may suggest that there are better ways to set the global estimate $\tilde{x}^{t+1}$. One way is to let the PS randomly select a worker to compute a local estimate $y_k^{t+}$ and set other workers in an idle mode, which saves local computation. Another way is to set $\tilde{x}^{t+1} = \frac{1}{n} \sum_{k=1}^n y_k^{t+}$ as the average of all local updates, which may potentially improve the performance. The specific choice in Alg.~\ref{alg:general} allows us to provide theoretical analysis for all the \texttt{LocalUpdate} rules considered in this paper. 

\end{remark}

	\subsection{Assumptions}Throughout, we invoke one or several of the following standard assumptions of the loss function in the convergence analysis.
	\begin{assumption}[Smoothness]
	    \label{assum:l_smooth}
	    The sample loss $\ell_z(\cdot)$ is $L$-smooth, i.e., the gradient of $\nabla\ell_z(\cdot)$ is $L$-Lipschitz for all $z \in \cM$.
	\end{assumption}
	\begin{assumption}[Convexity]
	    \label{assum:l_convex}
	    The sample loss $\ell_z(\cdot)$ is convex for all $z \in \cM$.
	\end{assumption}
	\begin{assumption}[Strong Convexity]
	    \label{assum:f_sc}
	    The empirical risk $f(\cdot)$ is $\sigma$-strongly convex.
	\end{assumption}
		
When $f$ is strongly convex, the condition number of $f$ is defined as $\kappa : = L/\sigma$. Denote the unique minimizer and the optimal value of $f(x)$ as 
$$x^* := \argmin_{x\in\mbR^d} f(x),\quad \quad f^* := f(x^*).$$

As it turns out, the smoothness of the deviation $f_k-f$ between the local loss function $f_k$ and the global loss function $f$ plays a key role in the convergence analysis, as it measures the balancedness between local data in a simple and intuitive manner. We refer to this as the ``distributed smoothness''. In some cases, a weaker notion called
 \emph{restricted smoothness} is sufficient, which is defined below. 
 
\begin{definition}[Restricted Smoothness]
\label{def:rsc}
A differentiable function $f: \mbR^d \mapsto \mbR$ is called $c$-restricted smooth with regard to $x^{*}$ if
$\norm{\nabla f(x^*) - \nabla f(y)} \le c \norm{x^*-y}$, for all $ y \in \mbR^d$.
\end{definition}
The restricted smoothness, compared to standard smoothness, fixes one of the arguments to $x^*$, and is therefore a much weaker requirement.  
  The following assumption quantifies the distributed smoothness using either restricted smoothness or standard smoothness.

	\edef\oldassumption{\the\numexpr\value{assumption}+1}

    \setcounter{assumption}{0}
    \renewcommand{\theassumption}{\oldassumption\alph{assumption}}
    \begin{assumption}[Distributed Restricted Smoothness] 
      	\label{assum:f_restricted}
	    The deviation $f-f_k$ is $c_k$-restricted smooth with regard to $x^*$ for all $1\leq k\leq n$.
    \end{assumption}
    
    \begin{assumption}[Distributed Smoothness]
        \label{assum:f_nrestricted}
	    The deviation $f-f_k$ is $c_k$-smooth for all $1\leq k\leq n$.
    \end{assumption}
 
It is straightforward to check that $c_k \leq L$ for all $1\leq k\leq n$. If all the data samples are generated following certain statistical distribution in an i.i.d. fashion, one can further link the distributed smoothness to the local sample size $N/n$, where $c_k$ decreases with the increase of $N/n$, see e.g. \cite{Shamir2013,fan2019communication} for further discussion.

\begin{remark}
	We provide a toy example to illustrate the difference between Assumptions~\ref{assum:f_restricted} and \ref{assum:f_nrestricted}. Let $x \in \mathbb{R}$ and $\ell_z(x) = L_{\delta(z)}(x) + \frac{\sigma}{2}x^2$, where $L_{\delta(z)}(\cdot)$ denotes Huber loss with $\delta(z) > 0$. We set  $\delta(z)= 0.99$ for  half of $z\in\mathcal{M}$, and  $1$ for the rest. Therefore,  $f(x) = \frac12(L_{0.99}(x)+L_1(x))+\frac{\sigma}{2}x^2$	
with $x^* = 0$. Now, consider some local loss function
 $f_k(x) = (\frac12 + \epsilon_k) L_{0.99}(x) + (\frac12 - \epsilon_k) L_1(x) + \frac{\sigma}{2}x^2$ for some $\epsilon_k$. We can verify that Assumption \ref{assum:f_restricted} holds with $c_k = |\epsilon_k|/100$ while Assumption \ref{assum:f_nrestricted} holds with $c_k = | \epsilon_k|$. Therefore, the distributed smoothness parameter can be much smaller when invoking the weaker Assumption~\ref{assum:f_restricted}.
\end{remark}

    \let\theassumption\origtheassumption

		\section{Convergence in the Strongly Convex Case} \label{sec:main_results}

 In this section, we describe three variance-reduced routines for \texttt{LocalUpdate} used in Alg.~\ref{alg:general}, namely SVRG \cite{Johnson2013}, SARAH \cite{Nguyen2017,nguyen2019optimal}, and MiG \cite{zhou2018simple}, and analyze their convergence when $f(\cdot)$ is strongly convex, respectively.

	\subsection{Distributed SVRG (D-SVRG)}
	\label{sec:result_dsvrg}
 The \texttt{LocalUpdate} routine of D-SVRG is described in Alg.~\ref{alg:dissvrg_main}. Theorem~\ref{thm:dissvrgII} provides the convergence guarantee of D-SVRG as long as the distributed restricted smoothness parameter is small enough.
	\begin{theorem}[D-SVRG]\label{thm:dissvrgII}
	    Suppose that Assumptions \ref{assum:l_smooth}, \ref{assum:l_convex} and \ref{assum:f_sc} hold, and Assumption~\ref{assum:f_restricted} holds with $c_k \leq c < \sigma/4$. With $m   = O(\kappa (1-4c/\sigma)^{-2})$ and proper step size $\eta$, the iterates of D-SVRG satisfy
    	\[
    	    \ex{f(\tx^{t+1}) - f^*} <  \frac{\sigma- 2c}{2(\sigma-3c)}  \ex{f(\tx^t) - f^*}.
    	\]
The communication and runtime complexities of finding an $\epsilon$-optimal solution (in terms of function value) are $\mathcal{O}\paren{\zeta^{-1}\log(1/\epsilon)}$ and 
     $   	\mathcal{O}\paren{(N/n + 
    	\zeta^{-2}\kappa)\zeta^{-1}\log(1/\epsilon)}$ respectively, 
    	where $\zeta = 1 - 4c/\sigma$. 
	\end{theorem}
		
	Theorem \ref{thm:dissvrgII} establishes the linear convergence of function values in expectation for D-SVRG, as long as the parameter $c$ is sufficiently small, e.g. $c<\sigma/4$. From the expressions of communication and runtime complexities, it can be seen that the smaller $c$, the faster D-SVRG converges -- suggesting that the homogeneity of distributed data plays an important role in the efficiency of distributed optimization. When $c$ is set such that $c/\sigma$ is bounded above by a constant smaller than $1/4$, i.e. $\zeta=\mathcal{O}(1)$, the runtime complexity becomes $\mathcal{O}\paren{\paren{N/n+\kappa}\log(1/\epsilon)}$, which improves the counterpart of SVRG $\mathcal{O}\paren{\paren{N+\kappa}\log(1/\epsilon)}$ in the centralized setting.

\begin{remark}
The above \texttt{LocalUpdate} routine corresponds to the so-called Option II (w.r.t. setting $y_k^{t,+}$ as uniformly at random selected from previous updates) specified in \cite{Johnson2013}. Under similar assumptions, we also establish the convergence of D-SVRG using Option I, where the output $y_k^{t,+}$ is set as $y_k^{t,m}$. In addition, D-SVRG still converges linearly in the absence of Assumption~\ref{assum:l_convex}. We leave these extensions in the appendix. 
\end{remark}
 
 \begin{algorithm}[t]
	\caption{\texttt{LocalUpdate via SVRG/SARAH}}
	\label{alg:dissvrg_main}
		\begin{algorithmic}[1]
		\STATE {\textbf{Input:} local data $\mathcal{M}_k$, $\tx^{t}$, $\nabla f(\tx^t)$;}
		\STATE {\textbf{Parameters:} step size $\eta$, number of iterations $m$;}
\STATE{Set $\subx_{k}^{t, 0}= \tx^{t}$, $v_k^{t, 0}= \nabla f(\tx^t)$;}
         \FOR{$s=0,...,m-1$}
        \STATE{Sample $z$  from $\cM_k$ uniformly at random;}\\
        \STATE{Compute $$
        v_k^{t,s+1}  =
        \begin{cases}
        \nabla \ell_z(\subx_{k}^{t,s+1}) -\nabla \ell_z(\tx^{t}) + \nabla f(\tx^{t}) & \texttt{SVRG}\\
        \nabla \ell_z(\subx_{k}^{t,s+1}) -\nabla \ell_z(\subx_{k}^{t,s}) + v_{k}^{t,s} & \texttt{SARAH}\\
        \end{cases}
        $$}
        \STATE{$\subx_{k}^{t, s+1}  = \subx_{k}^{t,s} - \eta v_k^{t,s}$;} 
        \ENDFOR
        \STATE{Set $\subx_{k}^{t, +}$ uniformly at random from $\subx_{k}^{t, 1},\ldots, \subx_{k}^{t, m}$.}
        	\end{algorithmic}
\end{algorithm}

	\subsection{Distributed SARAH (D-SARAH)}
	\label{sec:result_dsarah}
	 The \texttt{LocalUpdate}  of D-SARAH is also described in Alg. \ref{alg:dissvrg_main}, which is different from SVRG in the update of stochastic gradient $ v_k^{t,s}$, by using a recursive formula proposed in \cite{Nguyen2017}. Theorem~\ref{thm:dissarah} provides the convergence guarantee of D-SARAH as long as the distributed restricted smoothness parameter is small enough.

\begin{theorem}[D-SARAH]\label{thm:dissarah}
    Suppose that Assumptions \ref{assum:l_smooth}, \ref{assum:l_convex} and \ref{assum:f_sc} hold, and Assumption~\ref{assum:f_restricted} holds with $c_k \leq c < \sqrt{2}\sigma/4$. With $m = O(\kappa (1-2\sqrt{2}c/\sigma)^{-2})$ and proper step size $\eta$, the iterates of D-SARAH satisfy
    \[
        \ex{\norm{\nabla f(\tx^{t+1})}^2} < \frac{1}{2-8c^2/\sigma^2} \ex{\norm{\nabla f(\tx^t)}^2}.
    \]
   The communication and runtime complexities of finding an $\epsilon$-optimal solution (in terms of gradient norm) are $\mathcal{O}\paren{\zeta^{-1}\log(1/\epsilon)}$  and
   $\mathcal{O}\paren{(N/n + \zeta^{-2}\kappa)\zeta^{-1}\log(1/\epsilon)}$ respectively,
    where $\zeta = 1 - 2\sqrt{2}c/\sigma$.
\end{theorem}

	Theorem \ref{thm:dissarah} establishes the linear convergence of the gradient norm in expectation for D-SARAH, as long as the parameter $c$ is small enough. Similar to D-SVRG, a smaller $c$ leads to faster convergence of D-SARAH. When $c$ is set such that $c/\sigma$ is bounded above by a constant smaller than $\sqrt{2}/4$, the runtime complexity becomes $\mathcal{O}\paren{\paren{N/n+\kappa}\log(1/\epsilon)}$, which improves the counterpart of SARAH $\mathcal{O}\paren{\paren{N+\kappa}\log(1/\epsilon)}$ in the centralized setting. In particular, Theorem \ref{thm:dissarah}  suggests that D-SARAH may allow a larger $c$, compared with D-SVRG, to guarantee convergence.

\subsection{Distributed MiG (D-MiG)}	\label{sec:result_dmig}
 The \texttt{LocalUpdate} of D-MiG is described in Alg.~\ref{alg:dismig_main}, which is inspired by the inner loop of the MiG algorithm \cite{zhou2018simple}, a recently proposed accelerated variance-reduced algorithm. Compared with D-SVRG and D-SARAH, D-MiG uses additional information from the previous updates, according to Line 3-6 in Alg.~\ref{alg:dismig_main}. Theorem~\ref{thm:dismig} provides the convergence guarantee of D-MiG, as long as the distributed smoothness parameter is small enough.

 
	\begin{theorem}[D-MiG]\label{thm:dismig}
    Suppose that Assumptions \ref{assum:l_smooth}, \ref{assum:l_convex} and \ref{assum:f_sc} hold, and Assumption~\ref{assum:f_nrestricted} holds with $c_k \leq c < \sigma/8$. Let $w= (1+\eta \sigma)/(1+3\eta c)$. With $m = O(N/n) $ and proper step size $\eta$, the iterates of D-MiG achieve an $\epsilon$-optimal solution 
within a communication complexity of $\mathcal{O}((1+\sqrt{\kappa/(N/n)})\log(1/\epsilon))$ and runtime complexity of $\mathcal{O}((N/n + \sqrt{\kappa N/n})\log(1/\epsilon))$.
\end{theorem}

Theorem~\ref{thm:dismig} establishes the linear convergence of D-MiG under standard smoothness of $f-f_k$, in order to fully harness the power of acceleration. While we do not make it explicit in the theorem statement, the time complexity of D-MiG also decreases as $c$ gets smaller. Furthermore, the runtime complexity of D-MiG is smaller than that of D-SVRG/D-SARAH when $\kappa=\Omega (N/n)$. 

 \begin{remark}\label{remark:dig}
Theorem~\ref{thm:dismig} continues to hold for regularized empirical risk minimization, where the loss function is given as  $F(x) =f(x) + g(x)$, and $g(x)$ is a convex and non-smooth regularizer. In this case, line 11 in Alg.~\ref{alg:dismig_main} is changed to 
$$x_k^{t,s+1} = \argmin_x \left\{\frac{1}{2\eta}\norm{x-x_k^{t,s}}^2 + \left\langle v_{k}^{t,s},x\right\rangle+g(x)\right\}. $$
\end{remark}
    
    \begin{algorithm}[t]
		\caption{\texttt{LocalUpdate via MiG}}
		\label{alg:dismig_main}
		\begin{algorithmic}[1]
		\STATE {\textbf{Input:} local data $\mathcal{M}_k$, $\tx^{t}$, $\nabla f(\tx^t)$;}
		\STATE {\textbf{Parameters:} step size $\eta$, number of iterations $m$, and $w$;}
		\IF{$t = 0$}
		\STATE{Set $x_{k}^{t,0}=  \tx^{t}$;}
		\ELSE
		\STATE{Set $x_{k}^{t,0}=  x_{k}^{t-1,m}$;}  
		\ENDIF
             \FOR{$s=0,...,m-1$}
             \STATE{Set $y_k^{t,s} = (1-\theta)\tx^t + \theta x_k^{t,s}$;}
            \STATE{Sample $z$ from $\mathcal{M}_k$ uniformly at random, and set
             $v_{k}^{t,s} = \nabla \ell_z(\subx_{k}^{t,s})-\nabla \ell_z(\tx^t) + \nabla f(\tx^t);$}
             \STATE{$x_k^{t,s+1} = x_k^{t,s} - \eta v_k^{t,s}$;}
            \ENDFOR
            \STATE{Set $y_k^{t+} = \paren{\sum_{j=0}^{m-1} w^j}^{-1}\sum_{j=0}^{m-1} w^j y_k^{t,j+1}$.}
		\end{algorithmic}
\end{algorithm}


\section{Regularization Helps Unbalanced Data}
\label{sec:regSVRG}
 
 So far, we have established the convergence when the distributed smoothness is not too large. While it may be reasonable in certain  settings, e.g. in a data center where one has control over how to distribute data, it is increasingly harder to satisfy when the data are generated locally and heterogeneous across workers. However, when such conditions are violated, the algorithms might diverge. In this situation, adding a regularization term might ensure the convergence, at the cost of possibly slowing down the convergence rate. We consider regularizing the local gradient update of D-SVRG in Alg.~\ref{alg:dissvrg_main} as
\begin{equation}\label{eq:regularized_update}
v_k^{t,s+1}  = \nabla \ell_z(\subx_{k}^{t,s+1}) -\nabla \ell_z(\tx^{t}) + \nabla f(\tx^{t}) + \mu_k (\subx_{k}^{t,s+1} - \tx^t),
\end{equation}
where the last regularization term penalizes the proximity between the current iterates $\subx_{k}^{t,s+1}$ and the reference point $\tx^t$, where $\mu_k>0$ is the regularization parameter employed at the $\tth{k}$ worker. We have the following theorem.

\begin{theorem}[Distributed Regularized SVRG (D-RSVRG)]\label{thm:rsvrg}
Suppose that Assumptions \ref{assum:l_smooth}, \ref{assum:l_convex} and \ref{assum:f_sc} hold, and Assumption~\ref{assum:f_restricted} holds with $c_k < (\sigma+\mu_k)/4$. Let $\mu = \min_{1\leq k\leq n} \mu_k$. With proper $m$ and step size $\eta$, there exists some constant $0\le \nu < 1$ such that the iterates of D-RSVRG satisfy
\begin{align} \label{eq:rate_drsvrg}
\ex{f(\tx^{t+1}) - f^*} \le& \left(1- (1-\nu)\max\left\{\frac{\sigma}{L+\mu},1-\frac{\mu}{\sigma}\right\}\right) \nonumber \\
& \cdot \ex{f(\tx^t) - f^*},
\end{align}
and the runtime complexity of finding an $\epsilon$-optimal solution is bounded by
\small
\[
\mathcal{O}\paren{(N/n + 
	\zeta^{-2}\bar{\kappa})\zeta^{-1} \min\left\{\kappa + \frac{\mu}{\sigma},\frac{1}{\max\{ 1 -\mu/\sigma, 0 \} }\right\}\log(1/\epsilon)}, 
\]
\normalsize
where $\zeta = 1 - 4c/(\sigma+\mu)$ and $\bar{\kappa} = (L+\mu)/(\sigma +\mu)$.
\end{theorem}

Compared with Theorem~\ref{thm:dissvrgII}, 
Theorem~\ref{thm:rsvrg} relaxes the assumption $c_k<\sigma/4$ to $c_k<(\sigma+\mu_k)/4$, which means that by inserting a larger regularization $\mu_k$ to local workers that are not distributed smooth, i.e. those with large $c_k$, one can still guarantee the convergence of D-RSVRG. However, increasing $\mu$ leads to a slower convergence rate: a large $\mu=8L$ leads to an iteration complexity $\mathcal{O}(\kappa \log(1/\epsilon))$, similar to gradient descent.  Compared with SCOPE \cite{zhao2017scope} which requires a uniform regularization $\mu>L-\sigma$, our analysis applies tailored regularization to local workers, and potentially allows much smaller regularization to guarantee the convergence, since $c_k$'s can be much smaller than the smoothness parameter $L$.

\section{Convergence in the Nonconvex Case}
\label{sec:extensions}

In this section, we extend the convergence analysis of D-SARAH to handle nonconvex loss functions, since SARAH-type algorithms are recently shown to achieve near-optimal performances for nonconvex problems \cite{wang2018spiderboost,nguyen2019optimal,fang2018spider}. As a modification that eases the analysis,  we make every worker return $y_k^{t+} = y_k^{t,m}$ in line 9 of Alg.~\ref{alg:dissvrg_main}.
Our result is summarized in the theorem below.

\begin{theorem}[D-SARAH for non-convex losses]\label{thm:dissarah_nonc}
    Suppose that Assumption~\ref{assum:l_smooth} and Assumption~\ref{assum:f_nrestricted} hold with $c_k\leq c$. With the step size $\eta\le \frac{2}{L\paren{1+\sqrt{1+8m+4m(m-1)c^2/L^2}}}$, D-SARAH satisfies
    \[
        \frac{1}{Tm}\sum_{t=0}^{T-1}\sum_{s=0}^{m-1} \ex{\norm{\nabla f(y^{t,s}_{k(t)})}^2} \le\frac{2}{\eta Tm}\paren{f(\tx^0) - f^*},
    \]
    where $k(t)$ is the agent index selected in the $\tth{t}$ round for parameter update, i.e. $\tx^{t+1} = y^{t+}_{k(t)}$ (c.f. line 8 of Alg.~\ref{alg:general}). To find an $\epsilon$-optimal solution, the communication complexity is $\mathcal{O}\paren{1+ (\sqrt{n/N}+c/L)L/\epsilon}$, and the runtime complexity is $\mathcal{O} (N/n + (\sqrt{N/n} + N/n\cdot c/L) L/\epsilon )$ by setting $m=\Theta(N/n)$.
\end{theorem}

Theorem~\ref{thm:dissarah_nonc} suggests that D-SARAH converges as long as the step size is small enough. Furthermore, a smaller $c$ allows a larger step size $\eta$, and hence faster convergence. To gain further insights, assuming i.i.d. data at each worker, by concentration inequalities it is known that $c/L=\mathcal{O} (\sqrt{\log(N/n)/(N/n)})$ under mild conditions \cite{Mei2017}, and consequently, the runtime complexity of finding an $\epsilon$-accurate solution using D-SARAH is $\mathcal{O} (N/n + L \sqrt{\log(N/n)N/n}/\epsilon )$.  This is comparable to the best known result $\mathcal{O}(N + L\sqrt{N}/\epsilon)$ for the centralized SARAH-type algorithms in the nonconvex setting \cite{nguyen2019optimal,fang2018spider,wang2018spiderboost} up to logarithmic factors, where the data size is replaced from $N$ to the size of local data $N/n$ -- demonstrating again the benefit of data distribution.

	\section{Proofs of Main Theorems} \label{sec:outline}

In this section, we outline the convergence proofs of D-SVRG (Theorem~\ref{thm:dissvrgII}), D-RSVRG (Theorem~\ref{thm:rsvrg}), D-SARAH in the strongly convex (Theorem~\ref{thm:dissarah}) and nonconvex (Theorem~\ref{thm:dissarah_nonc}) settings, while leaving the details to Appendix~\ref{sec:proof_dmig}. The convergence proof of D-MiG (Theorem~\ref{thm:dismig}) is delegated to the supplemental materials due to space limits. Throughout this section, we simplify the notations $y_k^{t,s}$ and $v_k^{t,s}$ by dropping the $t$ superscript and the $k$ subscript, whenever the meaning is clear, since it is often sufficient to analyze the convergence of a specific worker $k$ during a single round.

\subsection{D-SVRG (Theorem~\ref{thm:dissvrgII})} \label{sec:proof_dsvrg}
We generalize the analysis of SVRG using the dissipativity theory in \cite{Hu2018} to the analysis of D-SVRG, which might be of independent interest. 	 Setting $\xi^s = y^s - x^*$, we can write the local update of D-SVRG via the following linear time-invariant system \cite{Hu2018}:
	\begin{equation*}
	\xi^{s+1} = \xi^s - \eta\left[\nabla \ell_{z}\paren{\subx^s}-\nabla \ell_z\paren{\subx^0} + \nabla f\paren{\subx^0}\right],
	\end{equation*}
where $z$ is selected uniformly at random from local data points $\mathcal{M}_k$, or equivalently,
	\begin{equation}
	\xi^{s+1} = A\xi^s + Bw^s,
	\end{equation}
	with $A = I_d$, $B = \begin{bmatrix}
	-\eta I_d & -\eta I_d
	\end{bmatrix}$, and
	\begin{equation*}
		w^s = \begin{bmatrix}
            	\nabla \ell_z\paren{\subx^s} - \nabla \ell_z\paren{x^*} \\
            	\nabla \ell_z\paren{x^*}-\nabla \ell_z\paren{\subx^0} + \nabla f\paren{\subx^0}
	           \end{bmatrix}.
	\end{equation*}
	Here, $I_d$ is the identity matrix of dimension $d$.


Dissipativity theory characterizes how ``the inputs'' $w^s$, $s = 0,1,2,\ldots$ drive the internal energy stored in the ``states'' $\xi^s$, $s = 0,1,2,\ldots$ via an energy function $V:\mbR^{d}\mapsto \mbR_+$ and a supply rate $S:\mbR^{d}\times \mbR^{2d}\mapsto \mbR$. For our purpose, it is sufficient to choose the energy function as 
$ V(\xi) = \| \xi\|_2^2$. By setting $S(\xi, w) = \sum_{j=1}^J \lambda_j S_j(\xi, w)$ as the supply rate, where
	\begin{equation}\label{eq:supply_rates}
	S_j(\xi, w) = [\xi^\top, w^\top]  X_j  [\xi^\top, w^\top]^\top ,
	\end{equation} 
	we have
	\begin{equation}\label{equ:dissip}
	V(\xi^{s+1}) \le \rho^2 V(\xi^k) + \sum_{j=1}^J \lambda_j S_j(\xi^s, w^s)
	\end{equation}
	for some $\rho \in (0,1]$ as long as there exist non-negative scalars $\lambda_j$ such that
	\begin{equation}\label{equ:lmi}
	\begin{bmatrix}
	A^\top A - \rho^2 I_d & A^\top  B \\
    	B^\top  A & B^\top  B \\
	\end{bmatrix}
	-\sum_{j=1}^J \lambda_j X_j \preceq 0.
	\end{equation}
	In fact, by left multiplying $[\xi_k^\top, w_k^\top]$ and right multiplying $[\xi_k^\top, w_k^\top]^\top$ to (\ref{equ:lmi}), we recover (\ref{equ:dissip}). 
	
	To invoke the dissipativity theory for convergence analysis, we take two steps. First, we  capture properties of the objective function such as strong convexity and co-coercivity using the supply rate functions by selecting proper matrices $X_j$ (c.f. Lemma~\ref{lem:svrgop2}). Next, we invoke \eqref{equ:lmi} to determine a valid combination of supply rate functions to guarantee the dissipativity inequality (c.f. Lemma~\ref{lem:SVRG_op2_pre}), which entails the rate of convergence. 
	

To proceed, we start by defining the supply rate functions in \eqref{eq:supply_rates}. Since $\xi^s \in \mbR^d$ and $w^s \in \mbR^{2d}$, we will write $X_j$ as $X_j = \bar{X}_j \otimes I_d$, where $\bar{X}_j\in \mbR^{3\times 3}$. Following \cite{Hu2018}, we consider the supply rates characterized by the following matrices:
		\begin{equation} \label{eq:supply_rates_optionII}
		\textstyle
    		\bar{X}_1 = 
    		\begin{bmatrix}
    		0 & 0 & 0\\
    		0 & 1 & 0\\
    		0 & 0 & 0\\
    		\end{bmatrix},
    		\bar{X}_2 = 
    		\begin{bmatrix}
    		0 & 0 & 0\\
    		0 & 0 & 0\\
    		0 & 0 & 1\\
    		\end{bmatrix},
    		\bar{X}_3 = 
    		\begin{bmatrix}
    		0 & -1 & -1\\
    		-1 & 0 & 0\\
    		-1 & 0 & 0\\
    		\end{bmatrix},
		\end{equation}
which correspond to the supply rates:
\begin{align*}
\ex{S_1}  &= \ex{\norm{\nabla \ell_z\paren{\subx^s}-\nabla \ell_z\paren{x^*}}^2}, \\
\ex{S_2}  &= \ex{\norm{\nabla \ell_z(x^*) - \nabla \ell_z(\subx^0) + \nabla f(\subx^0)}^2}, \\
\ex{S_3} & = - 2\ex{ \innprod{ \subx^s - x^*, \nabla \ell_z(\subx^s)-\nabla \ell_z(\subx^0) + \nabla f(\subx^0)}}.
\end{align*}
The following lemma, proved in Appendix \ref{proof:lem:svrgop2}, bounds the supply rates when $z$ is drawn from the local data $\mathcal{M}_k$.

	\begin{lemma}
		\label{lem:svrgop2}
		Suppose that Assumption \ref{assum:l_smooth}, \ref{assum:l_convex}, \ref{assum:f_sc} and \ref{assum:f_restricted} hold with $c_k\leq c$.
	For the supply rates defined in \eqref{eq:supply_rates_optionII}, we have 
		\begin{equation*}
		\begin{cases}
    		\ex{S_1} \le 2L\ex{f(\subx^s) - f^*} + 2cL\ex{\norm{\subx^s-x^*}^2};\\
    		\ex{S_2} \le 4L\ex{f(\subx^0) - f^*} + c(4L + 2c)\ex{\norm{\subx^0 - x^*}^2};\\ 
  \ex{S_3} \le -2\ex{f(\subx^s) - f^*} + 3c\ex{\norm{\subx^s - x^*}^2} + c\ex{\norm{\subx^0 - x^*}^2}. 
		\end{cases}
		\end{equation*}
	\end{lemma}
	 
The next lemma, proved in Appendix~\ref{proof:thm:SVRG_op2_pre}, details the dissipativity inequality when \eqref{equ:lmi} holds when setting $\rho=1$.
	\begin{lemma}
	    \label{lem:SVRG_op2_pre}
		Suppose that Assumption \ref{assum:l_smooth}, \ref{assum:l_convex}, \ref{assum:f_sc} and \ref{assum:f_restricted} hold with $c_k\leq c$. If there exist non-negative scalars $\lambda_j,\ j = 1, 2, 3$, such that 
		\begin{equation}\label{eq:svrg_inequality}
		\lambda_3 - L\lambda_1-(2L\lambda_1+3\lambda_3)c/\sigma > 0
		\end{equation} 
		and
		\begin{equation}
		\label{LMI:SVRG2}
		\begin{bmatrix}
		0 & \lambda_3 - \eta & \lambda_3 - \eta \\
		\lambda_3 - \eta & \eta^2 - \lambda_1 & \eta^2 \\
		\lambda_3 - \eta & \eta^2 & \eta^2 - \lambda_2
		\end{bmatrix}
		\preceq 0
		\end{equation}
		hold. Then D-SVRG satisfies  
		\begin{align}	\label{equ:svrg2}
		& \ex{f (\subx^+) - f^*}   \le \paren{\lambda_3 - L\lambda_1-(2L\lambda_1+3\lambda_3)c/\sigma}^{-1} \left[\frac{1}{\sigma m}+\frac{c}{\sigma}\paren{ (4L+2c)\lambda_2 + \lambda_3}+ 2L\lambda_2\right] \ex{f(\subx^0) - f^*} , 
		\end{align} 
		where the final output $\subx^+$ is selected from $y^1, \cdots, y^{m}$ uniformly at random.
	\end{lemma}
		
We are now ready to prove Theorem~\ref{thm:dissvrgII}. Set $\lambda_1 = \lambda_2 = 2\eta^2$, and $\lambda_3 = \eta$. We have \eqref{eq:svrg_inequality} holds with the step size $\eta < \dfrac{\sigma - 3c}{L(2\sigma+4c)} $, and (\ref{LMI:SVRG2}) holds since
$$	\begin{bmatrix}
	0 & 0 & 0 \\
	0 & -\eta^2 & \eta^2 \\
	0 & \eta^2 & -\eta^2
	\end{bmatrix}
	\preceq 0. $$
   	Applying Lemma~\ref{lem:SVRG_op2_pre} with the above choice of parameters, \eqref{equ:svrg2} can be written as
	\begin{align*} 
		\ex{f (\subx^+) - f^*} \le & \Bigg[  \dfrac{1/(\eta m) + c\paren{(8L+4c)\eta+1}+4L\eta}{\sigma(1-2L\eta )- c(4L\eta + 3) } \Bigg]   \cdot\ex{f(\subx^0) - f^*}.
	\end{align*}
	
When $c < \sigma /4$, by choosing $\eta = (1-4c/\sigma)(40L)^{-1}$, $m = 160\kappa(1-4c/\sigma)^{-2}$, a convergence rate no more than $\nu:= 1-\frac{1}{2}\cdot\frac{\sigma-4c}{\sigma-3c}$ is obtained. Therefore, after line 8 of Alg.~\ref{alg:general}, we have for D-SVRG,
\begin{align*}
		    \ex{f(\tx^{t+1})-f^*} &\le \frac{1}{n}\sum_{k=1}^n \ex{f(y_k^{t+})-f^*} \\
		    &\le \paren{1-\dfrac{1}{2}\cdot\dfrac{\sigma-4c}{\sigma-3c}}\ex{f(\tx^t)-f^*}.
\end{align*}
To obtain an $\epsilon$-optimal solution in terms of function value, we need $\mathcal{O}(\xi^{-1}\log(1 / \epsilon))$ communication rounds, where $\zeta = 1 - 4c/\sigma$. Per round, the runtime complexity at each worker is $\mathcal{O}(N/n + m)=\mathcal{O}(N/n + \zeta^{-2}\kappa )$, where the first term corresponds to evaluating the batch gradient over the local data in parallel, and the second term corresponds to evaluating the stochastic gradients in the inner loop. Multiplying this with the communication rounds leads to the overall runtime.


\subsection{D-RSVRG (Theorem~\ref{thm:rsvrg})}
\label{proof:rsvrg}

Consider an auxiliary sample function at the  $\tth{t}$ round and the $\tth{k}$ worker as
\[
\ell_{\mu_k}^t(x;z) = \ell(x;z) + \frac{\mu_k}{2}\norm{x - \tx^t}^2.
\]
This leads to the auxiliary local and global loss functions, respectively,
\[
f^t_i(x) = \frac{1}{|\mathcal{M}_i|}\sum_{z\in \mathcal{M}_i}\ell^t_{\mu_k}(x;z) = f_i(x) + \frac{\mu_k}{2}\norm{x - \tx^t}^2,  
\]
for $1\leq i \leq n$, and
\begin{equation}\label{eq:auxiliary_loss}
f^t(x) = \frac{1}{n}\sum_{i=1}^n f^t_i(x) = f(x) + \frac{\mu_k}{2}\norm{x - \tx^t}^2.
\end{equation}
Moreover, we have 
\begin{align*}
& \quad \nabla \ell_{\mu_k}^t(\subx_{k}^{t,s+1};z) -\nabla \ell_{\mu_k}^t(\tx^{t};z) + \nabla f^t(\tx^{t})\\
&=\nabla \ell(\subx_{k}^{t,s+1};z) -\nabla \ell(\tx^{t};z) + \nabla f(\tx^{t}) + \mu_k (\subx_{k}^{t,s+1} - \tx^t),
\end{align*}
which means that D-RSVRG performs in exactly the same way as the unregularized D-SVRG with the auxiliary loss functions $\ell_{\mu_k}^t$ in the $\tth{t}$ round. Note that $\ell_{\mu_k}^t$ is $(\mu_k + L)$-smooth and that $f^t$ is $(\mu_k + \sigma)$-strongly-convex, while the restricted smoothness of the $\tth{k}$ worker, $f^t - f_k^t = f - f_k$ remains unchanged. Applying Theorem~\ref{thm:dissvrgII}, we have 
 \begin{equation}\label{eq:auxiliary_descent}
\ex{f^t(\tx^{t+1}) - f^{t*}} < \nu \ex{f^t(\tx^t) - f^{t*}},
\end{equation}
where $f^{t*}$ is the optimal value of $f^t$, when Assumptions \ref{assum:l_smooth}, \ref{assum:l_convex} and \ref{assum:f_sc} hold, and Assumption~\ref{assum:f_restricted} holds, as long as $c_k < (\sigma+\mu_k)/4$, along with proper $m$ and step size $\eta$.

However, the definition of the regularized loss functions $\ell^t_{\mu}$ and ${f}^t$ rely on $\tx^t$, which changes over different rounds. Our next step is to relate the descent of $f^t$ to $f$. To this end, we have
\begin{align*}
& \ex{f(\tx^{t+1}) - f^*} \\
=& \ex{f^t(\tx^{t+1}) - f^{t*}} + f^{t*} - \ex{\frac{\mu}{2}\norm{\tx^t - \tx^{t+1}}^2} - f^* \\
<& \ex{ f^t(\tx^t) - f^{t*} } - (1-\nu)\ex{f^t(\tx^t) - f^{t*}} + f^{t*} - f^* \\
=& \ex{f(\tx^t) - f^*} - (1-\nu)\ex{f(\tx^t) - f^{t*} },
\end{align*}
where the second line uses \eqref{eq:auxiliary_loss}, the third line uses \eqref{eq:auxiliary_descent}, and the last line follows from $f(\tx^t) = f^t(\tx^t)$.

We can continue to bound $f(\tx^t) - f^{t*}$ in two manners. First, $f(\tx^t) - f^{t*} = f^t(\tx^t) - f^{t*} > \frac{1}{2(L+\mu)}\norm{\nabla f (\tx^t)}^2 > \frac{\sigma}{L+\mu}(f(\tx^t) - f^*)$. 
On the other hand, we have $f(\tx^t) -f^{t*} \ge f(\tx^t) - f^t(x^*) = f(\tx^t) - f(x^*) - \frac{\mu}{2}\norm{\tx^t - x^*}^2 \ge (1-\mu/\sigma)(f(\tx^t) - f^*)$. Thus \eqref{eq:rate_drsvrg} follows immediately by combining the above two bounds.

\subsection{D-SARAH in the Strongly Convex Case (Theorem~\ref{thm:dissarah})}

We motivate the convergence analysis by citing the following lemma from \cite{Nguyen2017}, which bounds the sum of gradient norm of iterations on a specific worker during $\tth{t}$ round:

\begin{lemma}\cite{Nguyen2017}
	\label{lem:grad_comp}
	Suppose that Assumption \ref{assum:l_smooth} holds, then 
	\begin{align*}
	& \sum_{s=0}^{m-1} \ex{\norm{\nabla f(y^s)}^2} \le \frac{2}{\eta} \ex{f(y^0) - f(y^{m})}  + \sum_{s=0}^{m-1} \ex{\norm{\nabla f(y^s) - v^s}^2} - (1-L\eta)\sum_{s=0}^{m-1}\ex{\norm{v^s}^2}.
	\end{align*}
\end{lemma}
We note that the first term on RHS can be neglected when $m$ is large; the second term measures the effect of biased gradient estimator and can thus be controlled with the step size $\eta$ and distributed smoothness; the third term can be dropped when $\eta \le 1/L$. A careful analysis leads to the following theorem, whose proof can be found in Appendix \ref{proof:thm:SARAH_pre}.

\begin{theorem}
    \label{thm:SARAH_pre}
	Suppose that Assumption \ref{assum:l_smooth}, \ref{assum:l_convex}, \ref{assum:f_sc} and \ref{assum:f_restricted} hold with $c_k\leq c$.
	 Then D-SARAH satisfies
	\begin{align*}
		     \paren{1-\frac{4c^2}{\sigma^2}} & \ex{\norm{\nabla f(\tx^{t+1})}^2}  	\le\paren{\frac{1}{\sigma\eta m} + \frac{4c^2}{\sigma^2}+\frac{2\eta L}{2-\eta L}}\ex{\norm{\nabla f(\tx^t)}^2}.
	\end{align*}
\end{theorem}
When $c < \frac{\sqrt{2}}{4} \sigma$, we can choose $\eta = \dfrac{2(1-8c^2/\sigma^2)}{(9-8c^2/\sigma^2)L}$ and $m = 2\kappa\dfrac{9-8c^2/\sigma^2}{(1-8c^2/\sigma^2)^2}$ in Theorem~\ref{thm:SARAH_pre}, leading to the following convergence rate:
	\[
\ex{\norm{\nabla f(\tx^{t+1})}^2}\\
		\le\frac{1}{2-8c^2/\sigma^2}\ex{\norm{\nabla f(\tx^t)}^2}.
	\]
    Consequently, following similar discussions as D-SVRG, the communication complexity of finding an $\epsilon$-optimal solution is $\mathcal{O}(\zeta^{-1}\log(1/\epsilon))$, and the runtime complexity is
$    \mathcal{O}\paren{(N/n + \zeta^{-2}\kappa)\zeta^{-1}\log(1/\epsilon)}$,    where $\zeta = 1 - 2\sqrt{2}c/\sigma$.

\subsection{D-SARAH in the Nonconvex Case (Theorem \ref{thm:dissarah_nonc})}

The convergence analysis in the nonconvex case is also based upon Lemma \ref{lem:grad_comp}. Due to lack of convexity, the tighter bound of the second term on RHS adopted in the analysis of Theorem~\ref{thm:SARAH_pre} is not available, so a smaller step size $\eta$ is needed to cancel out the second term and the third term, i.e., to make $$\sum_{s=0}^{m-1} \ex{\norm{\nabla f(y^s) - v^s}^2} - (1-L\eta)\sum_{s=0}^{m-1}\ex{\norm{v^s}^2} \leq 0.$$ Formally speaking, we have the following result, proved in Appendix~\ref{proof:thm:SARAH_nonc}.

\begin{theorem}
    \label{thm:SARAH_nonc}
    Suppose that Assumption \ref{assum:l_smooth} and \ref{assum:f_nrestricted} hold with $c_k\leq c$. By setting the step size
    \[
    \eta \le \frac{2}{L\paren{1+\sqrt{1+8(m-1)+4m(m-1)c^2/L^2}}},
    \]
     For a single outer loop of D-SARAH, it satisfies:
    \[
    	\sum_{s=0}^{m-1} \ex{\norm{\nabla f(y^s)}^2} \le \frac{2}{\eta} \ex{f(y^0) - f(y^{m})}.
    \]
\end{theorem}

By setting $\tx^{t+1} = y^{m}$, from the above theorem, we have 
\[
\sum_{s=0}^{m-1} \ex{\norm{\nabla f(y^s)}^2} \le \frac{2}{\eta} \ex{f(\tx^t) - f(\tx^{t+1})}.
\]
Hence, with $T$ outer loops, we have
\[
\frac{1}{Tm}\sum_{t=0}^{T-1}\sum_{s=0}^{m-1} \ex{\norm{\nabla f(y^{t,s}_{k(t)})}^2} \le\frac{2}{\eta Tm}\paren{f(\tx^0) - f^*}.
\]
where  $k(t)$ is the worker index selected in the $\tth{t}$ round for parameter update.
The communication complexity to achieve an $\epsilon$-optimal solution is
\begin{align*}
T =& \mathcal{O}\paren{1+ \frac{1}{\eta m \epsilon}} = \mathcal{O}\paren{1+\frac{\sqrt{m+m^2c^2/L^2}}{m}\cdot\frac{L}{\epsilon}} \\
=& \mathcal{O}\paren{1+\paren{\frac{1}{\sqrt{m}}+\frac{c}{L}}\cdot \frac{L}{\epsilon}},
\end{align*}
with the choice $\eta = \Theta\paren{\frac{1}{L\sqrt{m+m^2c^2/L^2}}}$. Per round, the runtime complexity at each worker is $\mathcal{O}(N/n + m)$. By choosing $m = \mathcal{O}(N/n)$,
we achieve the runtime complexity 
$$\tiny \mathcal{O}\paren{N/n + \paren{\sqrt{N/n} + N/n\cdot \frac{c}{L}}\frac{L}{\epsilon}}.$$

\begin{figure*}[!ht]
	\centering
	\begin{tabular}{ccc}
		\includegraphics[width=0.3\textwidth]{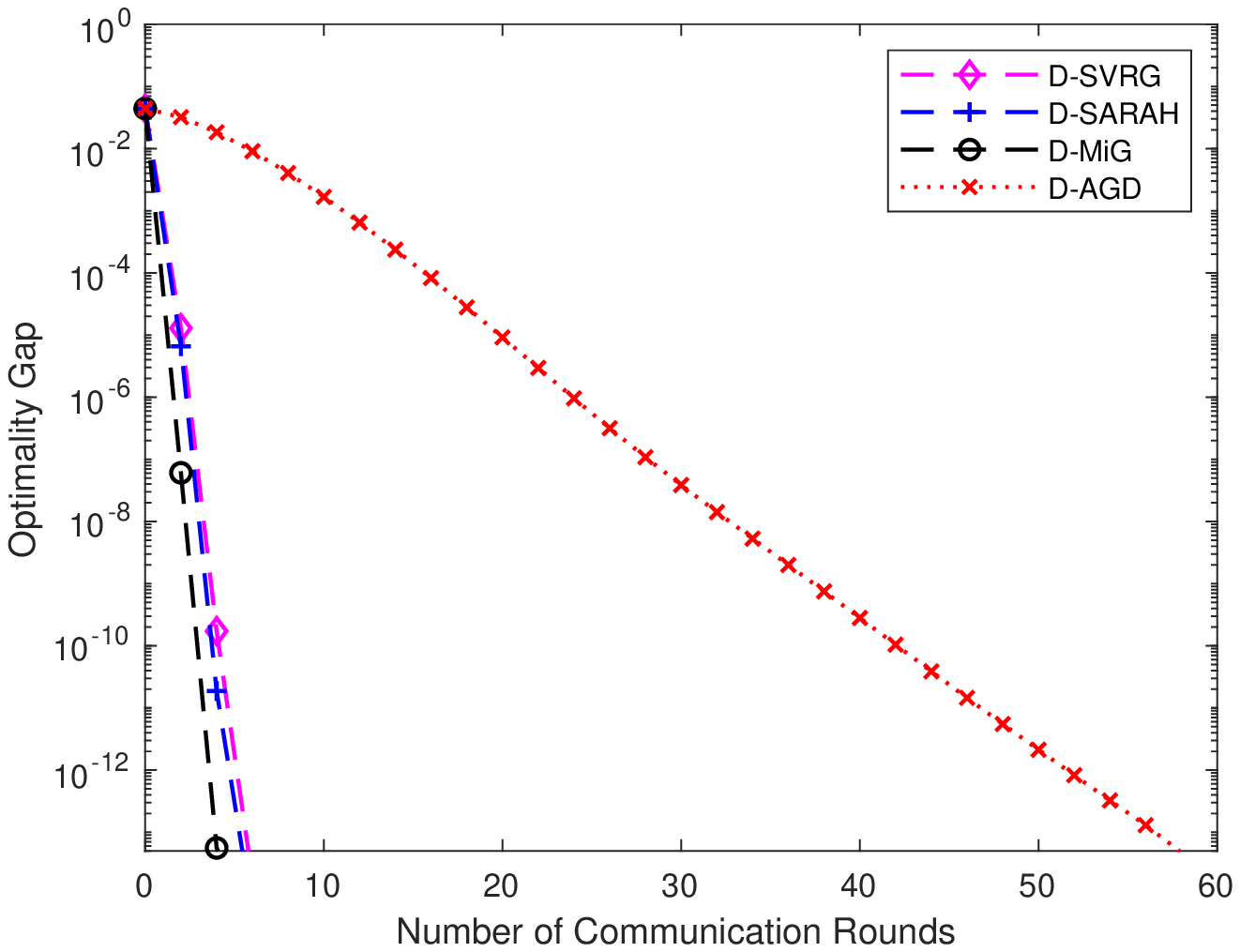} &
		\includegraphics[width=0.3\textwidth]{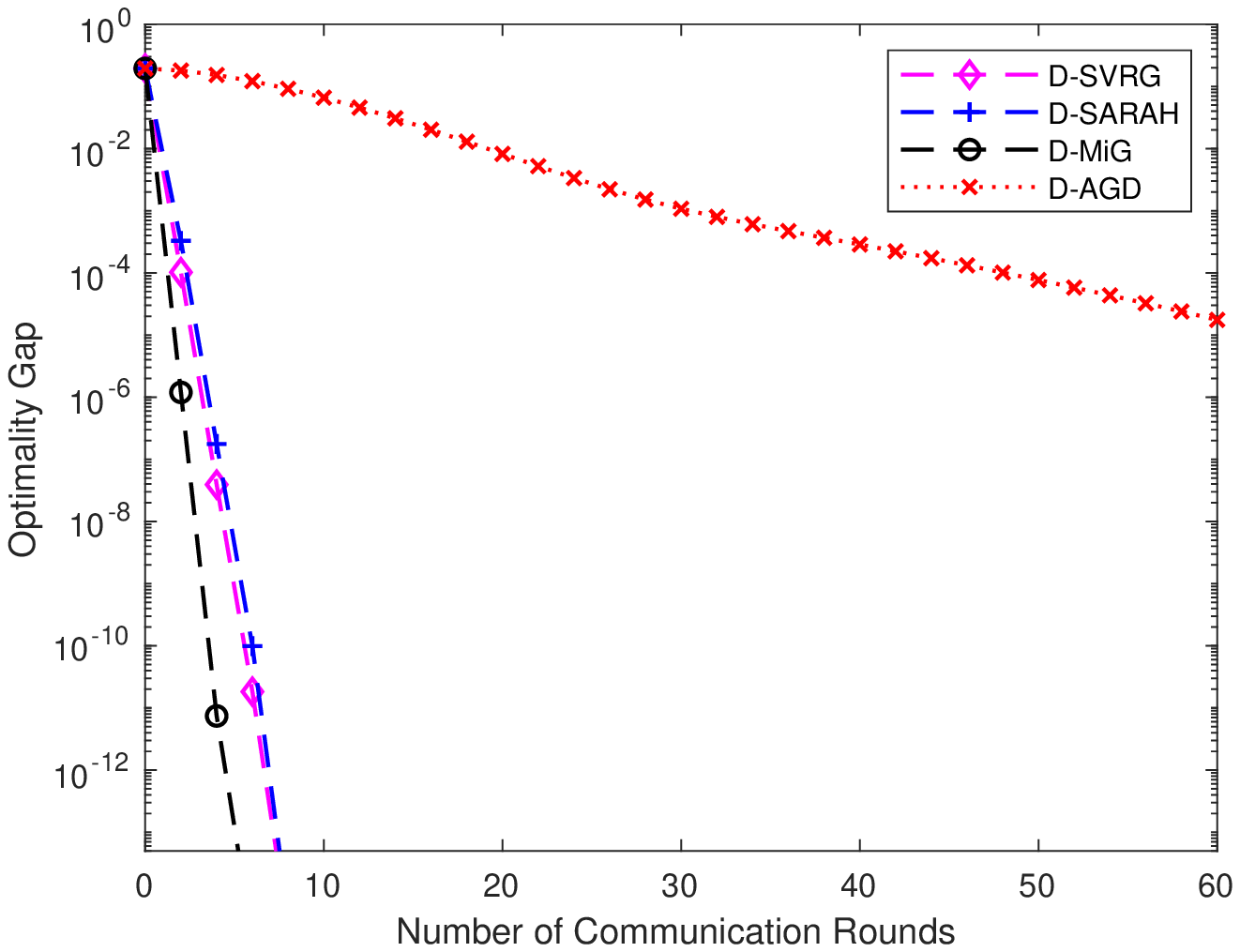} &
		\includegraphics[width=0.3\textwidth]{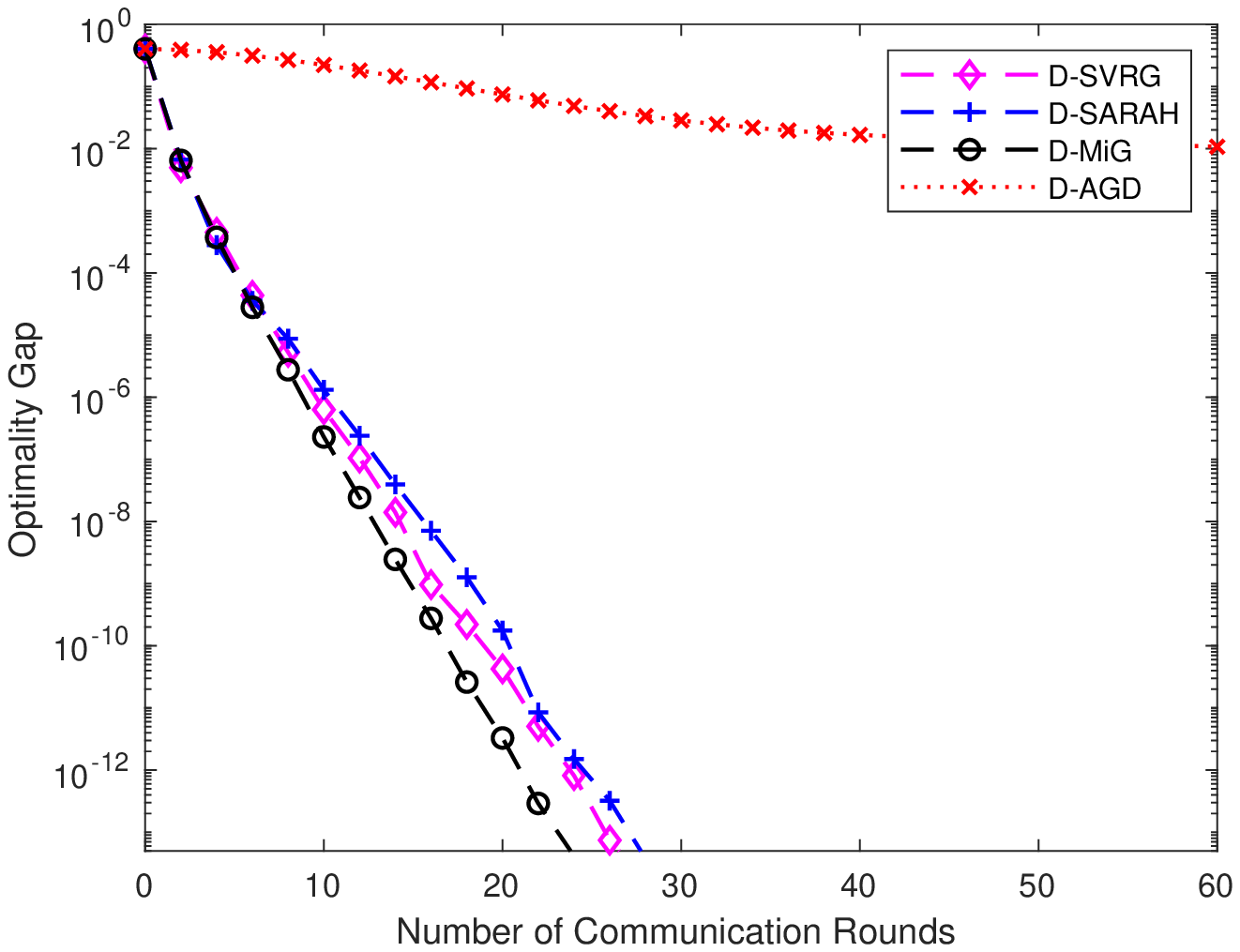}\\
		(a) $\lambda = N^{-0.5}$ (gisette) & (b) $\lambda =  N^{-0.75}$ (gisette) & (c) $\lambda =  N^{-1}$ (gisette) \vspace{0.02in}\\ 
		\includegraphics[width=0.3\textwidth]{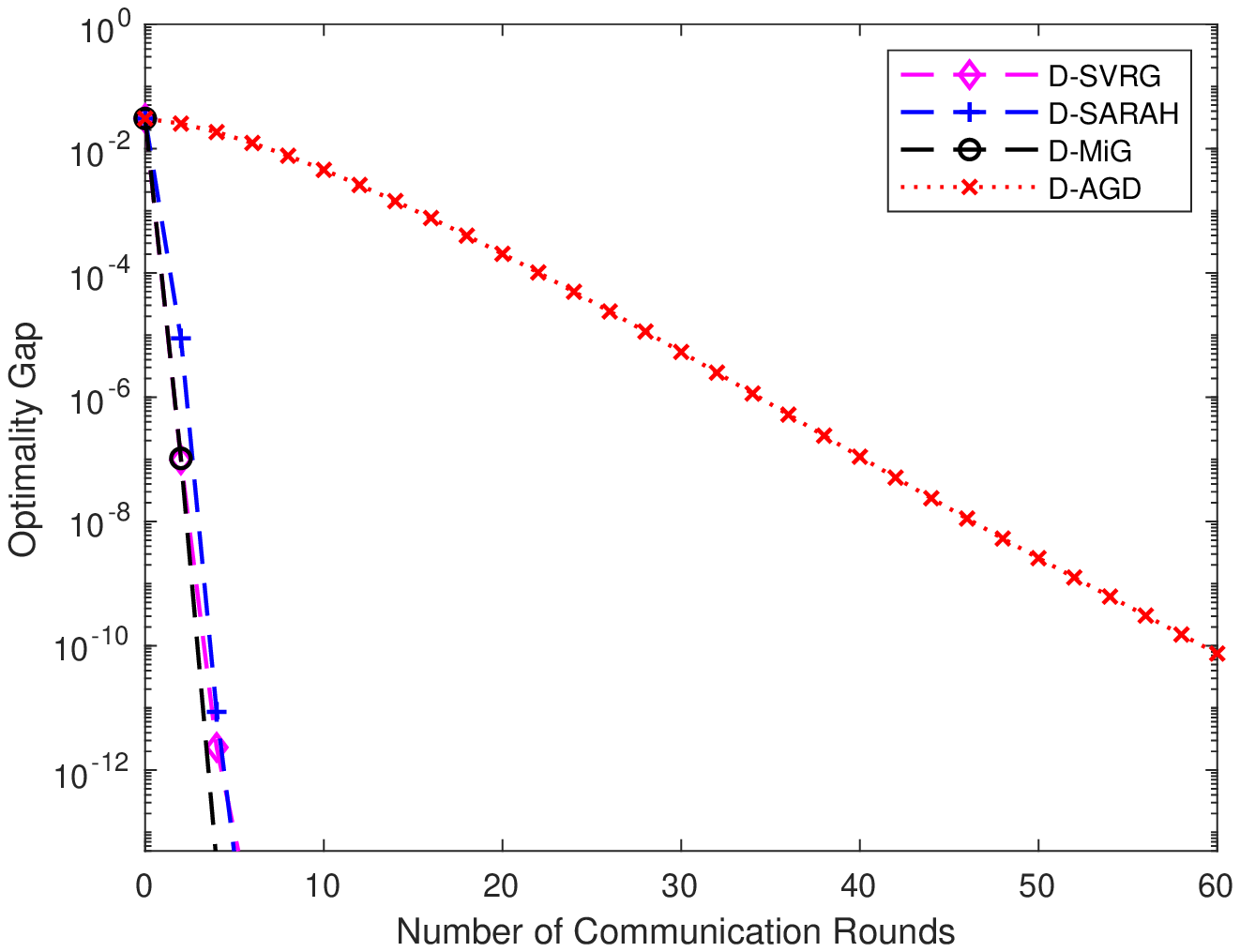} &
		\includegraphics[width=0.3\textwidth]{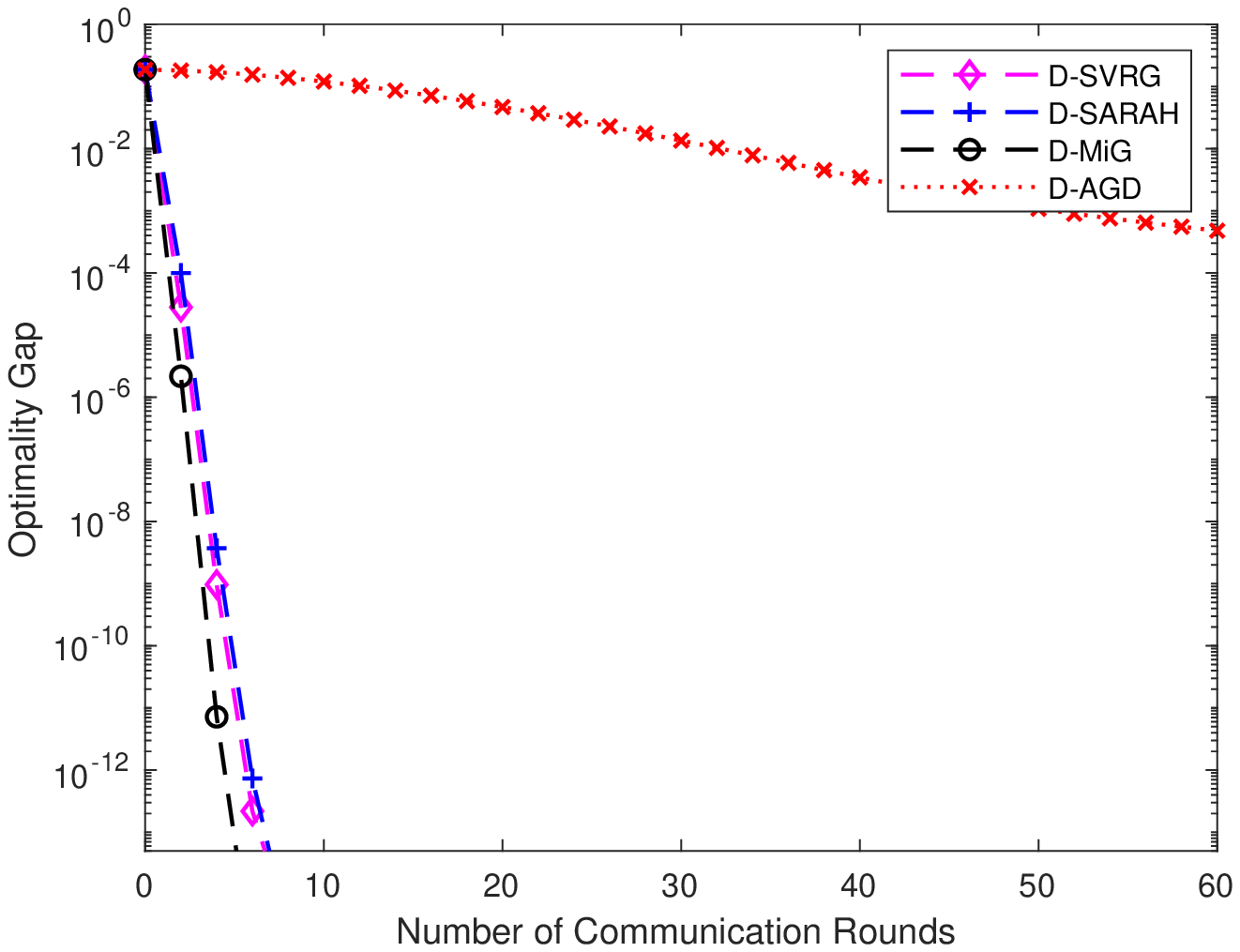} &
		\includegraphics[width=0.3\textwidth]{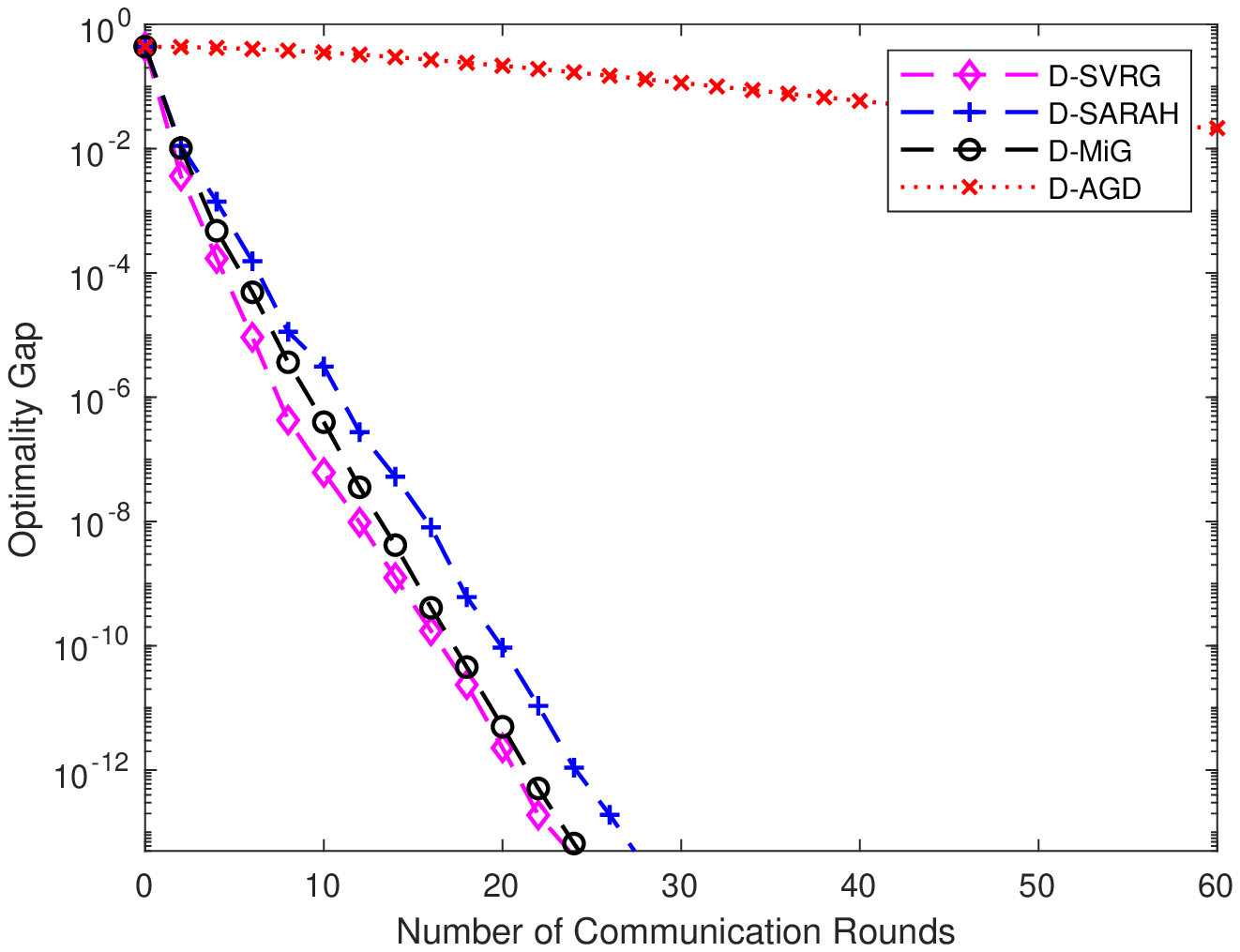} \\
		(d) $\lambda = N^{-0.5}$ (rcv1) & (e) $\lambda =  N^{-0.75}$ (rcv1) & (f) $\lambda =  N^{-1}$ (rcv1)
	\end{tabular}
	\caption{The optimality gap on $\ell_2$-regularized logistic regression with respect to the number of communication rounds with 4 workers using the gisette dataset (first row) and the rcv1 dataset (second row) under different conditioning for different algorithms. }
	\label{fig:kappa_data}
\end{figure*}

\section{Numerical Experiments}
\label{sec:numerical}

Thought the focus of this paper is theoretical, we illustrate the performance of the proposed distributed stochastic variance reduced algorithms in various settings as a proof-of-concept.

\subsection{Logistic regression in the strongly convex setting} 

Consider $\ell_2$-regularized logistic regression, where the sample loss is defined as
\begin{equation} \label{eq:logistic_loss}
\ell(x;z_i) = \log\paren{1 + \exp\paren{-b_i a_{i}^{\top} x}} + \frac{\lambda}{2}\norm{x}^2,
\end{equation} 
with the data $z_i = (a_i, b_i)\in \mathbb{R}^d \times \{\pm 1 \}$. We evaluate the performance on the gisette dataset \cite{guyon2005result} and the rcv1 dataset \cite{rcv1} by splitting the data equally to all workers. We scale the data according to $\max_{i\in [N]}\norm{a_{i}}^2 = 1$, so that the smoothness parameter is estimated as $L = 1/4 + \lambda$. We choose $\lambda = N^{-0.5}$, $N^{-0.75}$ and $N^{-1}$ respectively to illustrate the performance under different condition numbers. We use the optimality gap, defined as $f(\tx^t)-f^*$, to illustrate the convergence behavior.


\begin{figure*}[t]
	\centering
	\begin{tabular}{ccc}
		\includegraphics[width=0.3\textwidth]{{gisette_1_4agents}.eps} &
		\includegraphics[width=0.3\textwidth]{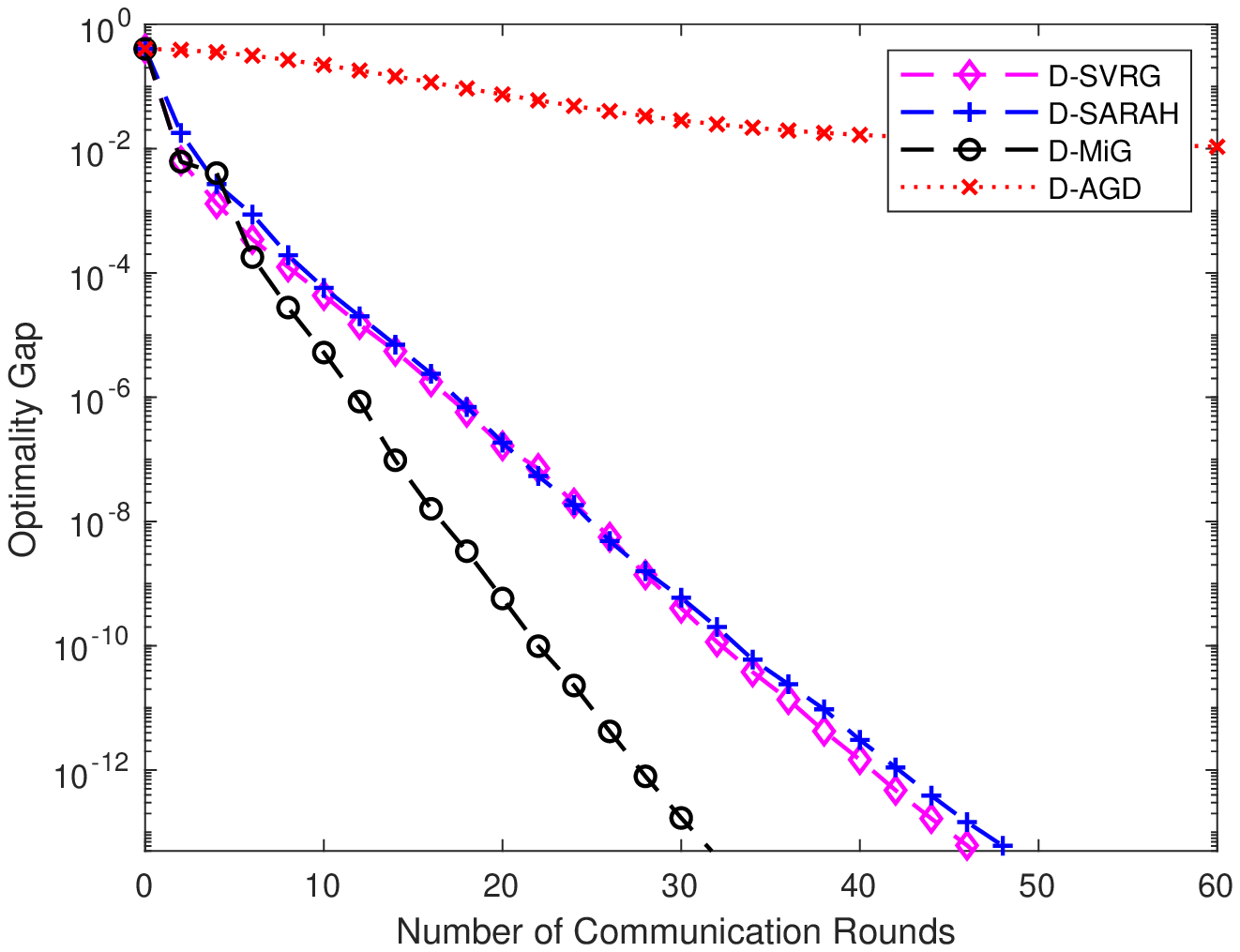} &
		\includegraphics[width=0.3\textwidth]{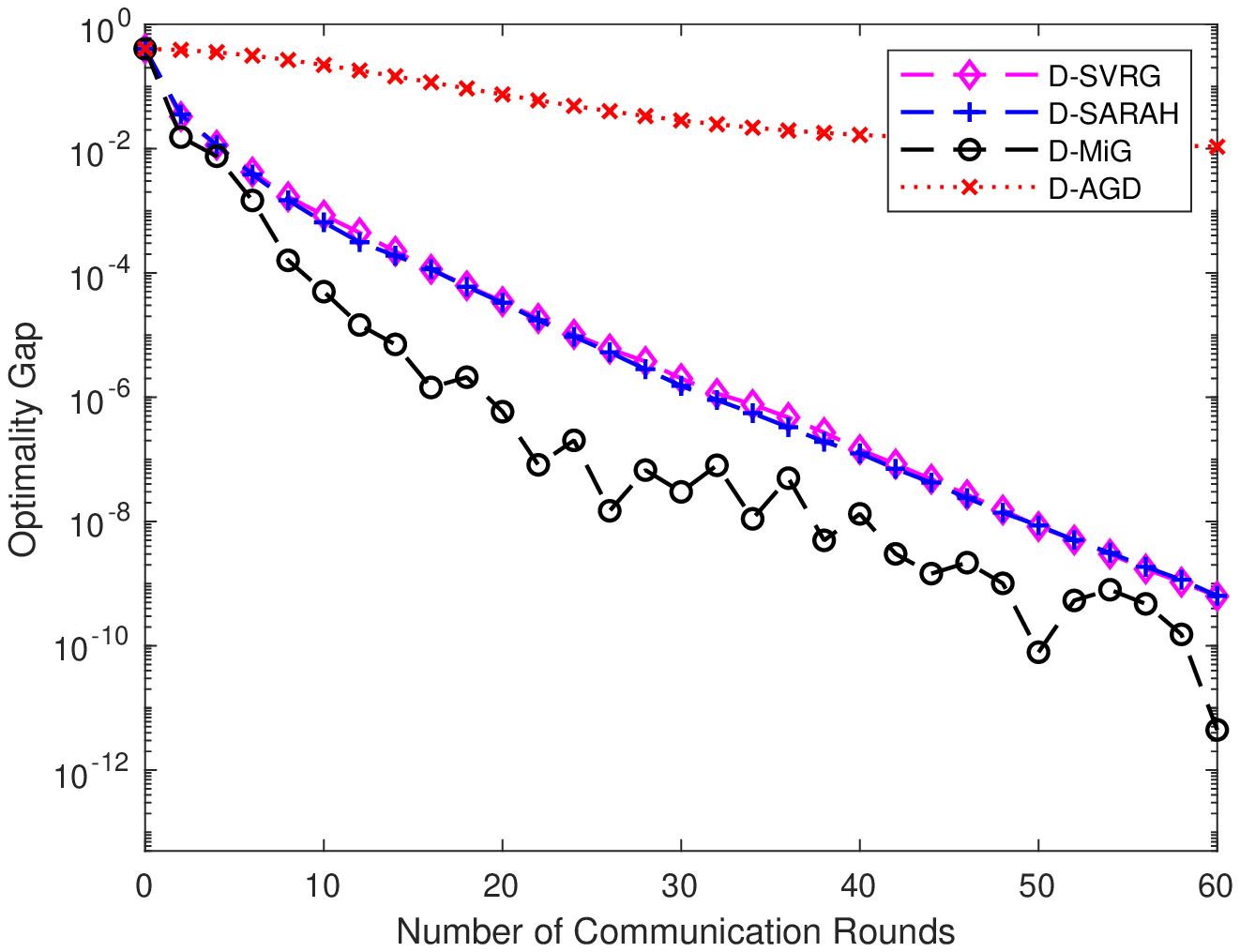}\\
		(a) $n=4$ (gisette) & (b) $n=8$ (gisette) & (c) $n=16$ (gisette)  \vspace{0.02in}\\ 
		\includegraphics[width=0.3\textwidth]{{rcv1_1_4agents}.eps} &
		\includegraphics[width=0.3\textwidth]{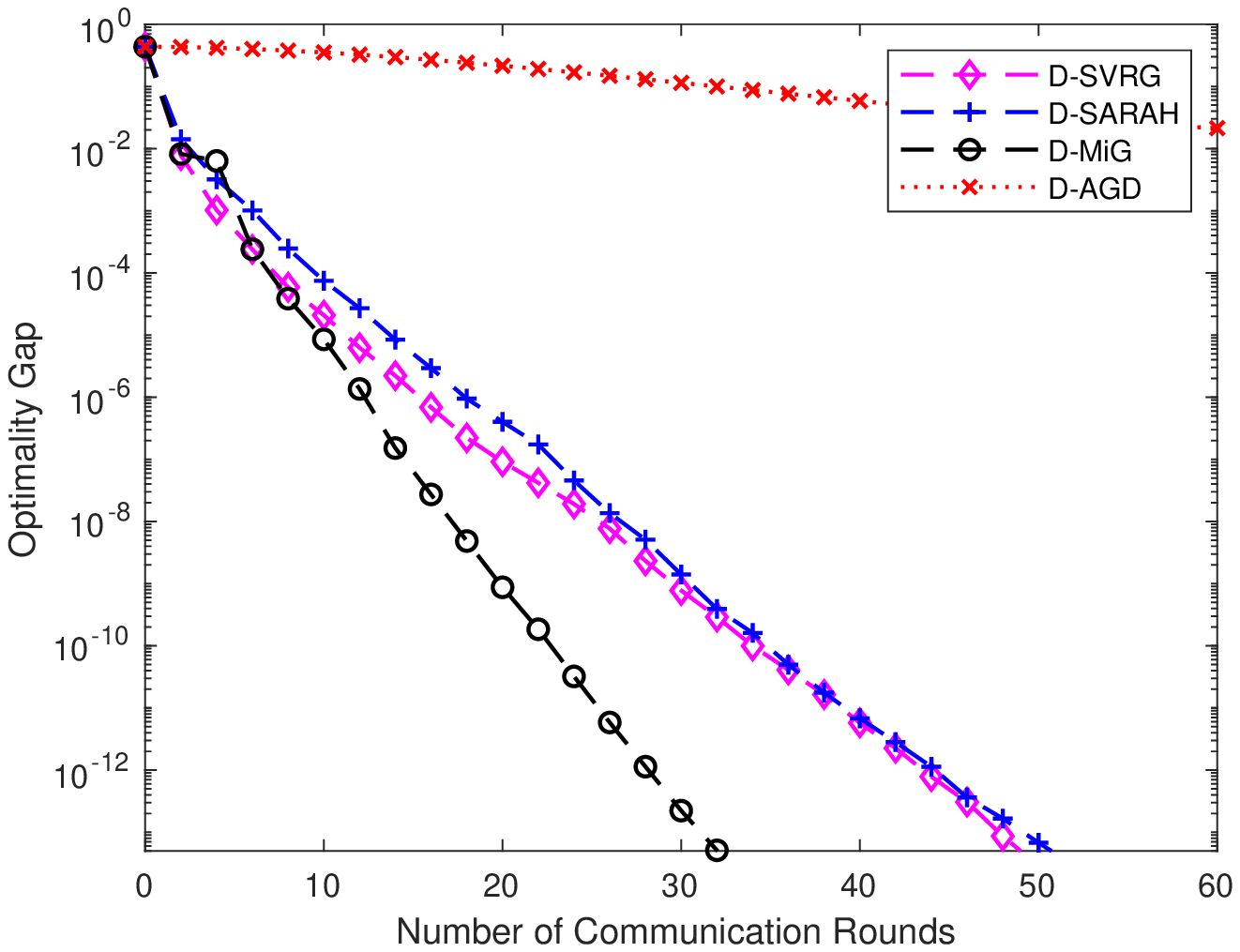} &
		\includegraphics[width=0.3\textwidth]{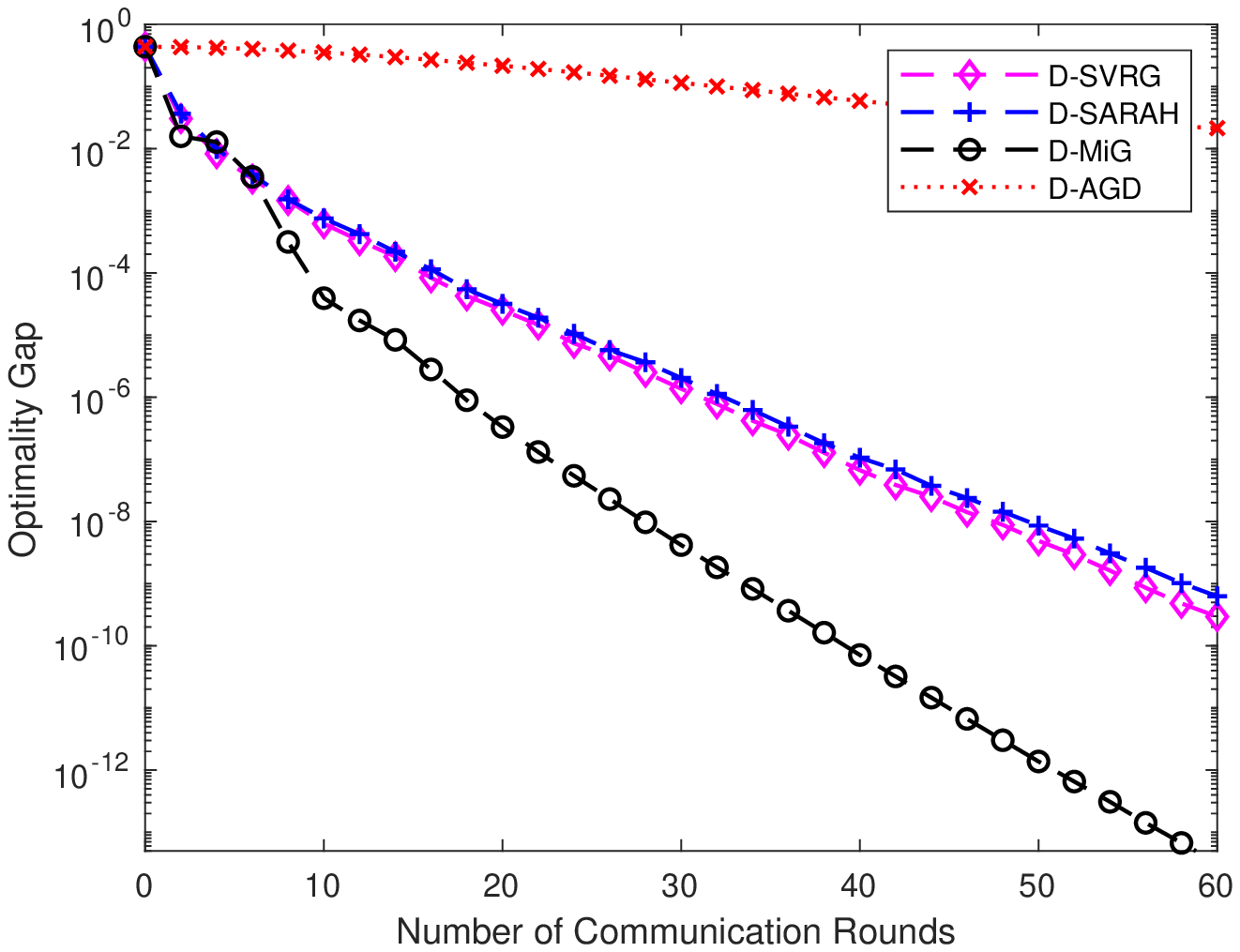} \\
		(d) $n=4$ (rcv1) & (e) $n=8$ (rcv1) & (f) $n=16$ (rcv1)
	\end{tabular}
	\caption{The optimality gap on $\ell_2$-regularized logistic regression with respect to the number of communication rounds with different number of workers using the gisette dataset (first row) and the rcv1 dataset (second row) for different algorithms when $\lambda=N^{-1}$.}
	\label{fig:n_kappa}
\end{figure*}

For D-SVRG and D-SARAH, the step size is set as $\eta = 1/(2L)$.
For D-MiG, although the choice of $w$ in the theory requires knowledge of $c$, we simply ignore it and set $w = 1+\eta \sigma$, $\theta = 1/2$ and the step size $\eta = 1/(3\theta L)$ to reflect the robustness of the practical performance to parameters. We further use $\tx^{t+1} = \frac{1}{n}\sum_{k=1}^n y_k^{t+}$ at the PS, which provides better empirical performance than the random selection rule in Alg.~\ref{alg:general}, as seen in Fig.~\ref{fig:gap_agents}. For D-AGD, the step size is set as $\eta = 1/L$ and the momentum parameter is set as $\frac{\sqrt{\kappa}-1}{\sqrt{\kappa}+1}$.
Following \cite{Johnson2013}, which sets the number of inner loop iterations as $m=2N$, we set $m \approx 2N/n$ to ensure the same number of total inner iterations. We note that such parameters can be further tuned to achieve better trade-off between communication cost and computation cost in practice.
\begin{figure}[h]
	\centering
	\vspace{-2ex}
	\includegraphics[width=.45\textwidth]{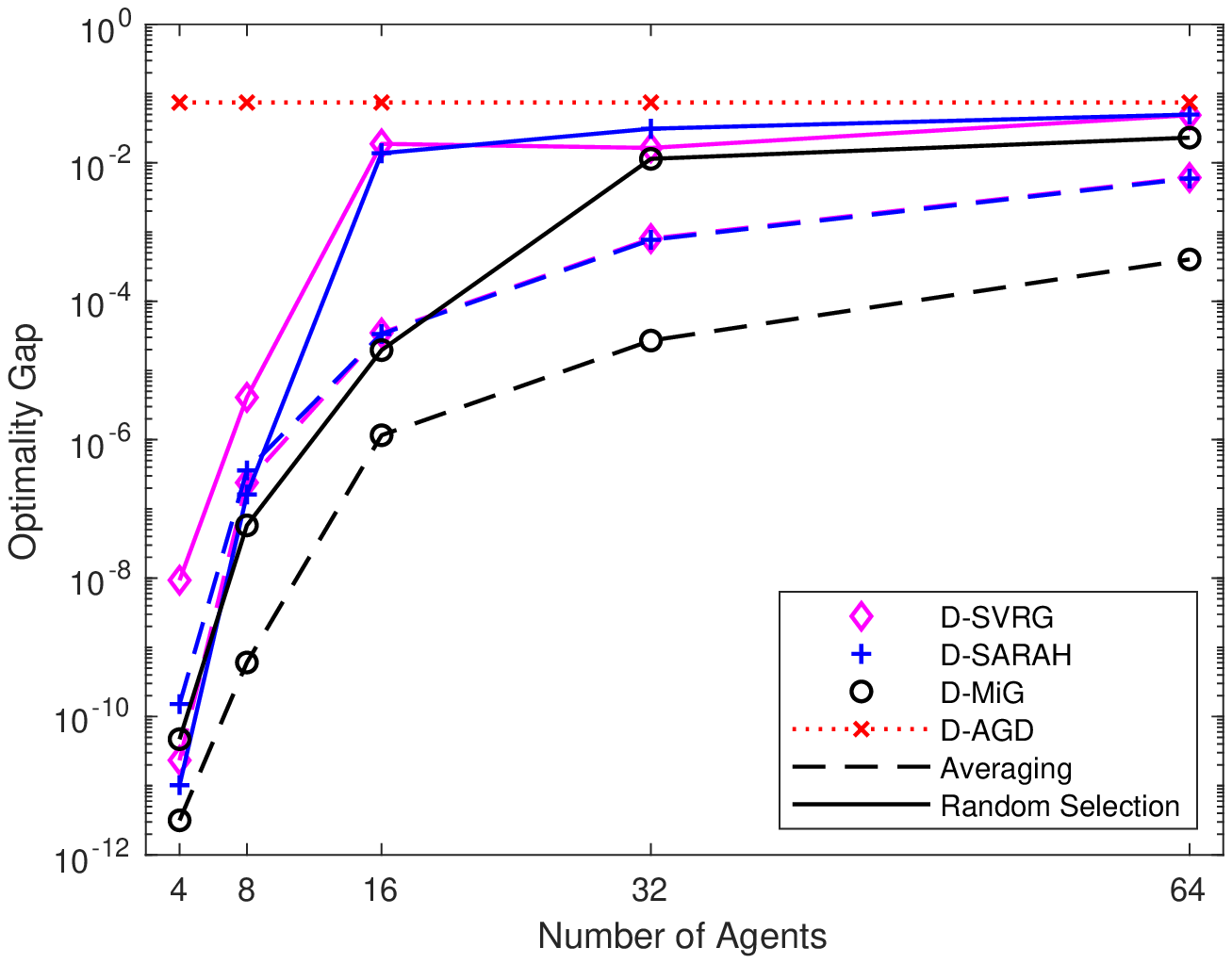}
	\caption{The optimality gap with respect to the number of agents after 20 communication rounds on the gisette dataset when $\lambda = N^{-1}$.}
	\label{fig:gap_agents}
\end{figure}

Fig.~\ref{fig:kappa_data}
illustrates the optimality gap of various algorithms with respect to the number of communication rounds with 4 local workers under different conditioning, and Fig.~\ref{fig:n_kappa} shows the corresponding result with different numbers of local workers when $\lambda=N^{-1}$. The distributed stochastic variance-reduced algorithms outperform  distributed AGD significantly. In addition, D-MiG outperforms D-SVRG and D-SARAH when the condition number is large. Fig.~\ref{fig:gap_agents} shows the optimality gap after a fixed  $20$ communication rounds with respect to the number of agents on the gisette dataset with $\lambda = N^{-1}$. We observe that the performance of distributed variance-reduced theorems degenerates as the number of agents grows, due to the shrinking size of local dataset, which leads to a larger distributed smoothness parameter $c$.

\subsection{Dealing with Unbalanced data.}
We justify the benefit of regularization by evaluating the proposed algorithms under unbalanced data allocation. We assign 50\%, 30\%, 19.9\%, 0.1\% percent of data to four workers, respectively, and set $\lambda=N^{-1}$ in the logistic regression loss \eqref{eq:logistic_loss}. To deal with unbalanced data, we perform the regularized update, given in \eqref{eq:regularized_update}, on the worker with the least amount of data, and keep the update on the rest of the workers unchanged. A similar regularized update can be conceived for D-SARAH and D-MiG, resulting in regularized variants, D-RSARAH and D-RMiG. While our theory does not cover them, we still evaluate their numerical performance. We properly set $\mu$ according to the amount of data on this worker as $\mu = 0.1/(0.1\%\cdot N)^{0.5}$. We set the number of iterations at workers $m=2N$ on all agents. Fig.~\ref{fig:nonbalance} shows the optimality gap with respect to the number of communication rounds for all algorithms. It can be seen that all unregularized methods fail to converge, and the regularized algorithms still converge, verifying the role of regularization in addressing unbalanced data.  
It is also worth mentioning that the regularization can be flexibly imposed depending on the local data size, rather than homogeneously across all workers.

\begin{figure}[h]
	\centering 
		\includegraphics[width=0.45\textwidth]{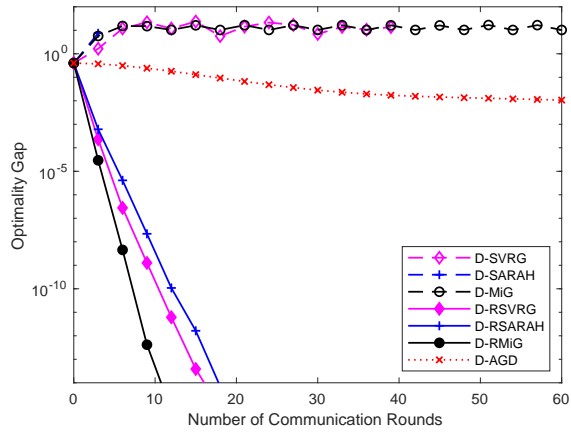}
		\caption{The optimality gap with respect to the number of communication rounds for  highly unbalanced data allocation. It can be seen that the regularized variants of distributed stochastic variance-reduced algorithms still converge will the unregularized ones no longer converge. }
		\label{fig:nonbalance}
\end{figure}

\subsection{Performance of D-SARAH in the nonconvex setting}

We follow the same setting as \cite{wang2018spiderboost} to evaluate D-SARAH and Distributed Gradient Descent (D-GD) on the gisette dataset with a nonconvex sample loss  function:
\[
\ell_{\mathrm{ncvx}}(x;z_i) = \log(1+\exp(-b_i a_i^\top x))+\lambda\sum_{j=1}^{d}x_j^2/(1+x_j^2),
\]
which consists of the logistic loss and a non-convex regularizer, where $x_j$ is the $j$th entry of $x$. The smoothness parameter of $\ell_{\mathrm{ncvx}}(x;z_i) $ can be estimated as $L = 1/4 + 2\lambda$. Fig.~\ref{fig:nonconvex} plots the squared norm the gradient $\| \nabla f(\tilde{x}^t)\|^2$ of D-SARAH and D-GD with respect to the number of communication rounds. It can be seen that D-SARAH achieves a much lower gradient norm than D-GD with the same number of communication rounds.

\begin{figure}[h] 
\centering
		\includegraphics[width=0.45\textwidth]{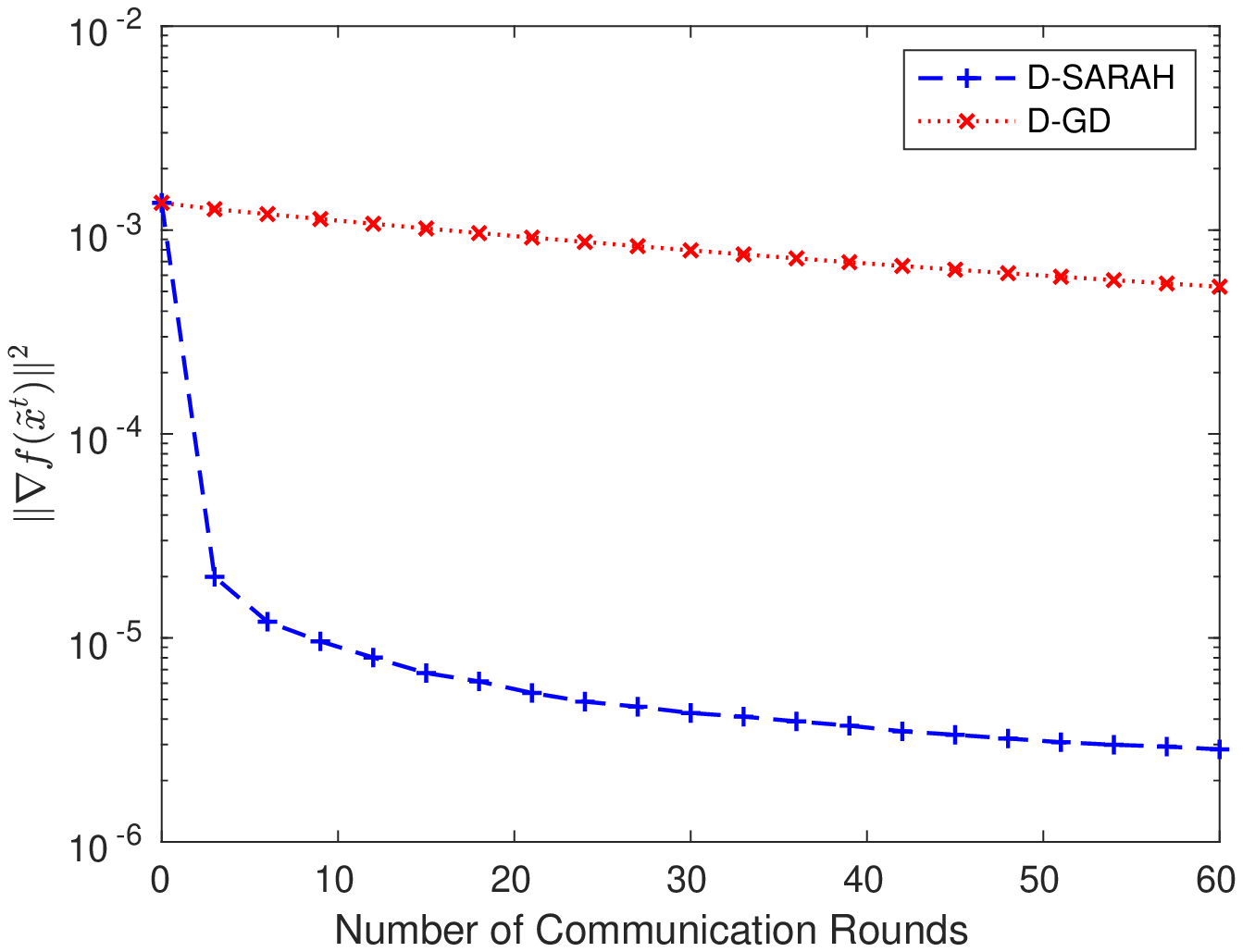}
		\caption{The squared norm of the gradient with respect to the number of communication rounds on the gisette dataset with 4 workers using a nonconvex loss function.}
		\label{fig:nonconvex}
\end{figure}

\section{Conclusions} \label{sec:conclusions}

In this paper, we have developed a convergence theory for a family of distributed stochastic variance reduced methods without sampling extra data, under a mild distributed smoothness assumption that measures the discrepancy between the local and global loss functions. Convergence guarantees are obtained for distributed stochastic variance reduced methods using accelerations and recursive gradient updates, and for minimizing both strongly convex and nonconvex losses. We also suggest regularization as a means of ensuring convergence when the local data are unbalanced and heterogeneous. We believe the analysis framework is useful for studying distributed variants of other stochastic variance-reduced methods such as Katyusha \cite{allen2017katyusha}, and proximal variants such as \cite{xiao2014proximal}.

\section*{Acknowledgements}
The results in this paper have been presented in part at the 2019 NeuriPS Workshop on Federated Learning for Data Privacy and Confidentiality. The work of S. Cen was partly done when visiting Microsoft Research Asia. The work of S. Cen and Y. Chi is supported in part by National Science Foundation under the grants CCF-1806154, CCF-1901199, and CCF-2007911, Office of Naval Research under the grants N00014-18-1-2142 and N00014-19-1-2404, and Army Research Office under the grant W911NF-18-1-0303.

 
\bibliography{distributed}
\bibliographystyle{alphaabbr}

\appendix

\section{Preliminary}

We first establish a lemma which will be useful later.
\begin{lemma}
\label{lem:rsp666}
When Assumptions~\ref{assum:l_smooth},\ref{assum:l_convex} and one of the distributed smoothness (Assumption \ref{assum:f_restricted} or \ref{assum:f_nrestricted}) hold, we have
\begin{align*}
&\mathbb{E}_z\left[\norm{\nabla \ell_z(x_1) - \nabla \ell_z(x_2)}^2\right]  \le 2LD_f(x_1, x_2)  + 
\begin{cases}
2cL\paren{\norm{x_1 - x^* }^2 + \norm{x_2- x^*}^2} &  \mbox{Assumption}~\ref{assum:f_restricted} \\
cL \norm{x_1 - x_2 }^2   &  \mbox{Assumption}~\ref{assum:f_nrestricted}
\end{cases},
\end{align*} 
for any $\tx$, where the expectation is evaluated over $z$, and $D_f(x_1, x_2) = f(x_1) - f(x_2) - \innprod{\nabla f(x_2), x_1 - x_2}$.
\end{lemma}
 
 \begin{proof}
Given $f$ is $L$-smooth and convex, the Bregman divergence $D_f(x_1, x_2)$
 is $L$-smooth and convex as a function of $x_1$. When Assumptions \ref{assum:l_smooth} and \ref{assum:l_convex} hold, we have
\begin{align*}
0 &= D_{\ell_z}(x_2, x_2)  \\ 
&\le D_{\ell_z} (x_1, x_2) - \frac{1}{2L}\norm{\nabla_{x_1} D_{\ell_z} (x_1,x_2)}^2\\
&= D_{\ell_z} (x_1,x_2) - \frac{1}{2L}\norm{\nabla \ell_z(x_1) - \nabla \ell_z(x_2)}^2.
\end{align*}
Averaging the above inequality over $z\in \cM_k$ gives
\begin{align} \label{equ:lem1step1}
2L \cdot & D_{f_k}(x_1, x_2)  \ge\mathbb{E}_z\left[\norm{\nabla \ell_z(x_1) - \nabla \ell_z(x_2)}^2\right]  .
\end{align}
To further bound the left-hand side, Assumption \ref{assum:f_restricted} allows us to compare $D_f$ and $D_{f_k}$:
\begin{align*}
	& |D_{f_k}(x_1, x_2) - D_{f}(x_1, x_2)| \\
	 =& \Big|D_{f-f_k}( x_1, x^*) + D_{f-f_k}(x^*, x_2) + \innprod{\nabla(f-f_k)(x^* - x_2), x_1 - x^*} \Big| \\
	 \le& \frac{c}{2}\norm{x_1 - x^*}^2 + \frac{c}{2}\norm{x^* - x_2}^2 + c\norm{x^*- x_2}\norm{ x_1 - x^*}\\
	 \le& c\paren{\norm{ x_1 - x^*}^2+\norm{x^* -x_2}^2}.
\end{align*}
Following similar arguments, using Assumption \ref{assum:f_nrestricted} we obtain a tighter bound by replacing $x^*$ with any $\tx$. In particular, setting $\tx = (x_1+x_2)/2$ we have
$$   |D_{f_k}(x_1, x_2) - D_{f}(x_1, x_2)| \leq c\norm{ x_1 - x_2}^2/2. $$
Combining the above estimates with  (\ref{equ:lem1step1}) proves the lemma.
\end{proof}

\section{Proof for D-SVRG}

\subsection{Proof of Lemma \ref{lem:svrgop2}}
\label{proof:lem:svrgop2}

For $\ex{S_1}$, we apply Lemma~\ref{lem:rsp666} directly:
\begin{align*} 
\ex{S_1}  &= \ex{\norm{\nabla \ell_z\paren{\subx^s}-\nabla \ell_z\paren{x^*}}^2}  \\
& \le 2L \ex{f (\subx^s) - f (x^*)} + 2cL \ex{\norm{\subx^s - x^*}^2},
\end{align*}
where the inequality follows from the definition of $D_f(y^s, x^*)$, and $\nabla f(x^*)=0$.
For $\ex{S_2}$, we have
\begin{align*}
&\ex{S_2} = \ex{\norm{\nabla \ell_z(x^*) - \nabla \ell_z(\subx^0) + \nabla f(\subx^0)}^2} \\
\le& 2 \ex{\norm{\nabla \ell_z(x^*) - \nabla \ell_z(\subx^0) - (\nabla f_k (x^*) - \nabla f_k(\subx^0))}^2}+2\ex{\norm{(\nabla f_k (x^*) - \nabla f_k(\subx^0)) + \nabla f(\subx^0)}^2} \\
\le& 2 \ex{\norm{\nabla \ell_z(x^*) - \nabla \ell_z(\subx^0)}^2}+2c^2\ex{\norm{\subx^0-x^*}^2} \\
\le& 4L \ex{f(\subx^0) - f(x^*)} + c(4L + 2c)\ex{\norm{\subx^0 - x^*}^2},
\end{align*} 
where the first inequality is due to $\norm{a+b}^2\le 2\norm{a}^2+2\norm{b}^2$, the second inequality follows from evaluating the expectation and Assumption~\ref{assum:f_restricted}, and the last step uses Lemma~\ref{lem:rsp666} again. 

For $\ex{S_3}$, we have
\begin{align*}
& \ex{S_3} = - 2\ex{ \innprod{ \subx^s - x^*, \nabla \ell_z(\subx^s)-\nabla \ell_z(\subx^0) + \nabla f(\subx^0)}} \\
=& 2\ex{ - \innprod{ \subx^s - x^*, \nabla f(\subx^s)}} -2\ex{\innprod{ \subx^s - x^*, \nabla (f - f_k)(\subx^0)-\nabla (f - f_k)(\subx^s)}}\\
\le& -2\ex{f(\subx^s) - f(x^*)} + 2c\ex{\norm{\subx^s-x^*}\paren{\norm{\subx^s-x^*}+\norm{\subx^0-x^*}}} \\
\le&  -2\ex{f(\subx^s) - f(x^*)} + 3c\ex{\norm{\subx^s - x^*}^2} + c\ex{\norm{\subx^0 - x^*}^2},
\end{align*}
 where the first inequality is obtained by applying Assumption~\ref{assum:l_convex}, Cauchy-Schwarz inequality and Assumption \ref{assum:f_restricted}.

\subsection{Proof of Lemma~\ref{lem:SVRG_op2_pre}}
\label{proof:thm:SVRG_op2_pre}
 	Setting $\rho = 1$ in \eqref{equ:lmi}, we have that it becomes equivalent to \eqref{LMI:SVRG2}, and consequently the dissipation inequality \eqref{equ:dissip} holds. In view of Lemma~\ref{lem:svrgop2}, it can be written as
\begin{align*}
	\ex{\norm{\subx^{s+1} - x^*}^2} 
		\le& (1+2cL\lambda_1+3c\lambda_3)\ex{\norm{\subx^s - x^*}^2}  - (2\lambda_3 - 2L\lambda_1)\ex{f(\subx^s) - f^*} \\
			& + 4L\lambda_2 \ex{f(\subx^0) - f^*} + c\left[(4L+2c)\lambda_2+ \lambda_3\right]\ex{\norm{\subx^0 - x^*}^2}.
\textsl{}		\end{align*}
	Since $f$ is $\sigma$-strongly convex, $\norm{y^s - x^*}^2 < \frac{2}{\sigma}\paren{f(y^s)-f^*}$, then the above 
\begin{align*}
&\ex{\norm{\subx^{s+1} - x^*}^2} \leq \ex{\norm{\subx^s - x^*}^2}  - \gamma_1 \ex{f(\subx^s) - f^*} + 4L\lambda_2 \ex{f(\subx^0) - f^*} + \gamma_2 \ex{\norm{\subx^0 - x^*}^2},
		\end{align*}
where $\gamma_1 = (2\lambda_3 - 2L\lambda_1 - (4L\lambda_1+6\lambda_3)c/\sigma)$ and $\gamma_2 = c\left[(4L+2c)\lambda_2+ \lambda_3\right]$	are introduced as short-hand notations. In addition, $\gamma_1>0$ by assumption. Telescoping the above inequality by summing over $s=0,\ldots, m-1$, we have
	\begin{align*}
\gamma_1 \sum_{s=0}^{m-1}\ex{f(\subx^s) - f^*} 	\le &  \paren{1 + \gamma_2 m}\ex{\norm{\subx^0 - x^*}^2} +4L\lambda_2 m \ex{f(\subx^0) - f^*}.
	\end{align*}
	 Note that the choice of $y^+$ implies
	\[
	\ex{f (\subx^+) - f^*} = \frac{1}{m}\sum_{s=0}^{m-1}\ex{f(\subx^s) - f^*}.
	\]
	Therefore,
	\begin{align*}
	\gamma_1 \ex{f(\subx^+)-f^*} \le & \paren{1/m + \gamma_2}\ex{\norm{\subx^0 - x^*}^2}  + 4L\lambda_2\ex{f(\subx^0) - f^*}.
	\end{align*}
	We obtain the final result by substituting $\ex{\norm{\subx^0 - x^*}^2} \le \frac{2}{\sigma}\ex{f\paren{\subx^0}-f^*}$ into the above inequality.

\subsection{Convergence of D-SVRG without Assumption \ref{assum:l_convex}}

When Assumption \ref{assum:l_convex} does not hold, we can still use similar arguments as the proof of Theorem~\ref{thm:dissvrgII} and establish the convergence of D-SVRG, though at a slower rate. Using the same supply rates \eqref{eq:supply_rates_optionII}, Lemma~\ref{lem:svrgop2} can be modified as below.
\begin{lemma}
	\label{lem:svrgop2_nc}
	Suppose that Assumption \ref{assum:l_smooth}, \ref{assum:f_sc} and \ref{assum:f_restricted} hold.
For the supply rates defined in \eqref{eq:supply_rates_optionII}, we have 
	\begin{equation*}
	\begin{cases}
		\ex{S_1} \le 2L^2\sigma^{-1} \ex{f(y^s) - f^*}\\
		\ex{S_2} \le 4(L^2 + c^2)\sigma^{-1}\ex{f(y^0) - f^*}\\ 
		    \ex{S_3} \le -2\ex{f(\subx^s) - f^*} + 3c\ex{\norm{\subx^s - x^*}^2}+ c\ex{\norm{\subx^0 - x^*}^2}
	\end{cases}
	\end{equation*}
\end{lemma}

\begin{proof}
    With $L$-smoothness of $\ell_z$ we have the following estimate:
    \[
        \mathbb{E}_z\left[\norm{\nabla \ell_z(y_1) - \nabla \ell_z(y_2)}^2\right] \le L^2\norm{y_1-y_2}^2.
    \]
    So we have
    \begin{equation*}
    \ex{S_1} = \ex{\norm{\nabla \ell_z\paren{\subx^s}-\nabla \ell_z\paren{x^*}}^2} \le L^2 \ex{\norm{\subx^s - x^*}^2} \le 2L^2\sigma^{-1} \ex{f(y^s) - f^*} 
    \end{equation*}
    and
    \begin{equation*}
    \begin{aligned}
    \ex{S_2} =& \ex{\norm{\nabla \ell_z(x^*) - \nabla \ell_z(\subx^0) + \nabla f(\subx^0)}^2} \\
    \le& 2 \ex{\norm{\nabla \ell_z(x^*) - \nabla \ell_z(\subx^0) - (\nabla f_k (x^*) - \nabla f_k(\subx^0))}^2}+2\ex{\norm{(\nabla f_k (x^*) - \nabla f_k(\subx^0)) + \nabla f(\subx^0)}^2} \\
    \le& 2 \ex{\norm{\nabla \ell_z(x^*) - \nabla \ell_z(\subx^0)}^2}+2c^2\ex{\norm{y^0-x^*}^2} \\
    \le& 2(L^2 + c^2)\ex{\norm{\subx^0 - x^*}^2} \le 4(L^2 + c^2)\sigma^{-1}\ex{f(y^0) - f^*}.
    \end{aligned}
    \end{equation*}
    The estimate of $\ex{S_3}$ is identical to that in Lemma~\ref{lem:svrgop2}.
\end{proof}

Following the same process in the proof of Lemma~\ref{lem:SVRG_op2_pre}  in Appendix~\ref{proof:thm:SVRG_op2_pre}, we have the following inequality with proper choices of $\lambda_1$, $\lambda_2$ and $\lambda_3$:
	\[
		\begin{aligned}
			\ex{\norm{\subx^{s+1} - x^*}^2} 
			&\le (1+3c\lambda_3)\ex{\norm{\subx^s - x^*}^2} - (2\lambda_3 - 2L^2\sigma^{-1}\lambda_1)\ex{f(\subx^s) - f^*} \\
			&\quad+ 4(L^2+c^2)\sigma^{-1}\lambda_2 \ex{f(\subx^0) - f^*} + c\lambda_3\ex{\norm{\subx^0 - x^*}^2} \\
			&\le \ex{\norm{\subx^s - x^*}^2}- (2\lambda_3 - 2L^2\sigma^{-1}\lambda_1 - 6c\lambda_3\sigma^{-1})\ex{f(\subx^s) - f^*} \\
			&\quad+ 4(L^2+c^2)\sigma^{-1}\lambda_2 \ex{f(\subx^0) - f^*} + c\lambda_3\ex{\norm{\subx^0 - x^*}^2} .
		\end{aligned}
	\]
By summing the inequality and letting $\lambda_1 = \lambda_2 = 2\eta^2$, $\lambda_3  = \eta$, we have
\[
	\begin{aligned}
	&2\eta(1 - 2L^2\sigma^{-1}\eta - 3c\sigma^{-1})\sum_{s=0}^{m-1}\ex{f(\subx^s) - f^*} \\
	&\quad\quad\le \paren{1 +  c\eta m}\ex{\norm{\subx^0 - x^*}^2} +8(L^2+c^2)\sigma^{-1}\eta^2 m \ex{f(\subx^0) - f^*} \\
	&\quad\quad\le \paren{2\sigma^{-1} + 2c\eta m\sigma^{-1}+8(L^2+c^2)\sigma^{-1}\eta^2 m}\ex{f(\subx^0) - f^*}.
	\end{aligned}
\]
Therefore, with $1 - 2L^2\sigma^{-1}\eta - 3c\sigma^{-1} > 0$ the following convergence bound can be established:
\[
    \ex{f\paren{\subx^+} - f^*} \le \frac{(\eta \sigma m)^{-1} + c \sigma^{-1}+4(L^2+c^2)\sigma^{-1}\eta}{1 - 2L^2\sigma^{-1}\eta - 3c\sigma^{-1}}\ex{f(\subx^0) - f^*}.
\]
By choosing $\eta = \dfrac{1-4c\sigma^{-1}}{40\kappa L}$, $m = \dfrac{160\kappa^2}{(1-4c\sigma^{-1})^2}$, we get a convergence rate no more than $1- \dfrac{1}{2}\cdot\dfrac{\sigma-4c}{\sigma-3c}$. Hence the overall time complexity to find an $\epsilon$-optimal solution is
\[
    \mathcal{O}\paren{\paren{N/n + \zeta^{-2}\kappa^2}\zeta^{-1}\log(1/\epsilon)}
\]
where $\zeta = 1-4c/\sigma$.

\subsection{Convergence of D-SVRG with Option I} \label{sec:proof_svrg_option1}

Another option for the output of the inner loops of SVRG, or the output of the local workers, is to output the last iterates, i.e. $y_k^{t+}=y_k^{t,m}$, which is called ``Option 1'' in  \cite{Johnson2013}. Here, we establish the convergence of D-SVRG using Option I in the following theorem. 
 
	\begin{theorem}[D-SVRG with Option I]
	\label{thm:DSVRG_op1}
	Suppose that Assumptions \ref{assum:l_smooth}, \ref{assum:l_convex} and \ref{assum:f_sc} hold, and Assumption~\ref{assum:f_restricted} holds with $c < \sigma/2$. With sufficiently large $m$ and sufficiently small step size $\eta$, there exists $0\le \nu < 1$ such that
	\[
	    \ex{\norm{\tx^{t+1} - x^*}^2} < \nu \ex{\norm{\tx^t-x^*}^2}.
	\]
	\end{theorem}

	Theorem \ref{thm:DSVRG_op1} indicates that the iterates $\tx^{t}$ of D-SVRG with Option I converge to the minimizer $x^*$ linearly in expectation as long as $c$ is sufficiently small.  The proof is outlined in Section~\ref{sec:proof_dsvrg}. By taking $m\to\infty$ and $\eta\to 0$, the rate approaches $\nu: = \frac{c}{2\sigma -3c} $, which suggests the algorithm admits faster convergence with the decrease of $c$, as expected.
	
Following \cite{Hu2018}, we consider the following four supply rates:
		\begin{equation} \label{eq:supply_rates_optionI}
		\begin{aligned}
    		&\bar{X}_1 = 
    		\begin{bmatrix}
        		0 & 0 & 0\\
        		0 & 0 & 0\\
        		0 & 0 & 1\\
    		\end{bmatrix}
    		,
    		\bar{X}_2 = 
    		\begin{bmatrix}
        		2\sigma & -1 & -1\\
        		-1 & 0 & 0\\
        		-1 & 0 & 0\\
    		\end{bmatrix},
    		\bar{X}_3 = 
    		\begin{bmatrix}
        		0 & -L & 0\\
        		-L & 2 & 0\\
        		0 & 0 & 0\\
    		\end{bmatrix}
    		,
    		\bar{X}_4 = -
    		\begin{bmatrix}
        		0 & 0 & 1\\
        		0 & 0 & 0\\
        		1 & 0 & 0\\
    		\end{bmatrix}.
		\end{aligned}
		\end{equation}
		
We have the following lemma which is proved at the end of this subsection.
	\begin{lemma}
		\label{lem:svrgop1}
		Suppose that Assumptions \ref{assum:l_smooth}, \ref{assum:l_convex}, \ref{assum:f_sc} and \ref{assum:f_restricted} hold.
	For the supply rates defined in \eqref{eq:supply_rates_optionI}, we have 
		\begin{equation*}
		\begin{cases}
    		\ex{S_1} \le (L^2 + 2cL -\sigma^2)\ex{\norm{\subx^0 - x^*}^2}\\
    		\ex{S_2} \le 3c\ex{\norm{\subx^s - x^*}^2} + c\ex{\norm{\subx^0 - x^*}^2}\\ 
    		\ex{S_3} \le 0\\
    		\ex{S_4} \le c\ex{\norm{\subx^s - x^*}^2} + c\ex{\norm{\subx^0 - x^*}^2}\\
		\end{cases}
		\end{equation*}
	\end{lemma}
	
	Therefore, by choosing $P = I_d$, $\lambda_1 = 2\eta^2$, $\lambda_2 = \eta -L\eta^2$, $\lambda_3 = \eta^2$, $\lambda_4 = L\eta^2$, $\rho^2 = 1-2\sigma(\eta-L\eta^2)$, the condition \eqref{equ:lmi} holds:
	\begin{equation*}
	\eta^2 
	\begin{bmatrix}
	0 & 0 & 0 \\
	0 & -1 & 1 \\
	0 & 1 & -1 \\
	\end{bmatrix}
	\preceq 0.
	\end{equation*}
	This immediately leads to the inequality \eqref{equ:dissip} which reads as:
	\begin{equation}
	\label{equ:svrgres1}
	\begin{aligned}
	\ex{\norm{\subx^{s+1}-x^*}^2} \le &(1-(2\sigma-3c)(\eta-L\eta^2) +cL\eta^2)\ex{\norm{\subx^s-x^*}^2} \\
	&+(2\eta^2(L^2 + 2cL-\sigma^2)+c\eta)\ex{{\norm{\subx^0 - x^*}}^2}.
	\end{aligned}
	\end{equation}
	Let $\tilde{\rho}^2 = (1-(2\sigma-3c)(\eta-L\eta^2) +cL\eta^2)$. Telescoping the inequality over $t=0,1,\cdots,m-1$ leads to
	\begin{equation*}
	\begin{aligned}
	\ex{\norm{\subx^m-x^*}^2} \le& \paren{\tilde{\rho}^{2m} + \frac{2\eta(L^2 +2cL-\sigma^2)+c}{(2\sigma-3c)(1-L\eta) -cL\eta}}\cdot\ex{\norm{\subx^0 - x^*}^2}
	\end{aligned}
	\end{equation*}
	Note that when $m \to \infty$ and $\eta \to 0$, the rate becomes $\nu:=\frac{c}{2\sigma-3c}$, hence we need $c<\sigma/2$ to get a rate $\nu<1$. We have
	\begin{equation*}
	    \begin{aligned}
	    \ex{\norm{\tx^{t+1}-x^*}^2} & \le \frac{1}{n}\sum_{k=1}^n \ex{\norm{y_k^{t,m} - x^*}^2} \\
	     & \le \frac{1}{n}\sum_{k=1}^n \nu\ex{\norm{y_k^{t,0} - x^*}^2} = \nu \ex{\norm{\tx^t - x^*}^2}.
	    \end{aligned}
	\end{equation*}

\begin{proof}[Proof of Lemma \ref{lem:svrgop1}]
    The following inequalities can be viewed as combinations of standard inequalities in convex optimization (co-coercivity, etc) and the characterization of restricted smoothness.
	\begin{equation*}
	\begin{aligned}
		\ex{S_1} =& \ex{\norm{\nabla \ell_z\paren{x^*}-\nabla \ell_z\paren{\subx^0} + \nabla f\paren{\subx^0}}^2} \\
		=& \ex{\norm{\nabla \ell_z\paren{x^*}-\nabla \ell_z\paren{\subx^0}}^2} + 2\ex{\innprod{\nabla \ell_z\paren{x^*}-\nabla \ell_z\paren{\subx^0}, \nabla f\paren{\subx^0}}} + \ex{\norm{\nabla f\paren{\subx^0}}^2} \\
		\le& L^2 \ex{\norm{x^* - \subx^0}^2} + 2\ex{\innprod{ \nabla(f-f_k)(\subx^0) - \nabla(f-f_k)(x^*), \nabla f\paren{\subx^0}}} - \ex{\norm{\nabla f\paren{\subx^0}}^2}\\
		\le&(L^2 + 2cL -\sigma^2)\ex{\norm{\subx^0 - x^*}^2},
	\end{aligned}
	\end{equation*}
	and
	\begin{equation*}
	\begin{aligned}
		\ex{S_2} =& 2\ex{\sigma \norm{\subx^s-x^*}^2 - \innprod{ \subx^s - x^*, \nabla \ell_z\paren{\subx^s}-\nabla \ell_z\paren{\subx^0} + \nabla f\paren{\subx^0}}} \\
		=& 2\ex{\sigma \norm{\subx^s-x^*}^2 - \innprod{ \subx^s - x^*, \nabla f(\subx^s)}}-2\ex{\innprod{ \subx^s - x^*, \nabla (f - f_k)(\subx^0)-\nabla (f - f_k)(\subx^s)}}\\
		\le& 2c\ex{\norm{\subx^s-x^*}\paren{\norm{\subx^s-x^*}+\norm{\subx^0-x^*}}} \\
		\le& 3c\ex{\norm{\subx^s - x^*}^2} + c\ex{\norm{\subx^0 - x^*}^2}.
	\end{aligned}
	\end{equation*}
	    $\ex{S_3} \le 0$ is simply the restatement of $L$-smoothness of $\ell(\cdot, z)$, $z\in \cM$. 
	\begin{equation*}
	\begin{aligned}
		\ex{S_4} =& 2\ex{\innprod{ \subx^s - x^*, \nabla(f-f_k)(\subx^0) - \nabla(f-f_k)(x^*)}} \\
		\le & 2c\ex{\norm{\subx^s - x^*}\norm{\subx^0 - x^*}}\\
		\le& c \ex{\norm{\subx^s - x^*}^2} + c\ex{\norm{\subx^0 - x^*}^2}.
	\end{aligned}
	\end{equation*}
\end{proof}

\section{Proof for D-SARAH} \label{sec:proof_dsarah}

\subsection{Proof of Theorem \ref{thm:SARAH_pre}}
\label{proof:thm:SARAH_pre}

To begin, we cite   a supporting lemma from \cite{Nguyen2017}.

\begin{lemma}\cite{Nguyen2017}
	\label{lem:v_concaten}
	Suppose Assumption \ref{assum:l_smooth} and \ref{assum:l_convex} hold and $\eta < 2/L$, then 	
	\[
		\ex{\norm{v^s - v^{s-1}}^2} \le \frac{\eta L}{2-\eta L} \left[ \ex{\norm{v^{s-1}}^2}  - \ex{\norm{v^s}^2} \right].
	\] 
\end{lemma}

We also present a new lemma below with the proof given in Appendix \ref{proof:lem:trace_concaten}.
\begin{lemma}
	\label{lem:trace_concaten}
	The update rule of D-SARAH satisfies
\begin{align*}
			&\ex{\norm{\nabla f(y^0) - \nabla f_k(y^0) + \nabla f_k(y^s) - v^s}^2} \\ =& \sum_{j=1}^s \ex{\norm{v^j-v^{j-1}}^2} - \sum_{j=1}^s\ex{\norm{\nabla f_k(y^j) - \nabla f_k(y^{j-1})}^2}.
		\end{align*}
\end{lemma}

\begin{proof}[Proof of Theorem \ref{thm:SARAH_pre}]
By combining Lemmas \ref{lem:v_concaten} and \ref{lem:trace_concaten}, we have
\begin{align}\label{eq:term1}
& \ex{\norm{\nabla f(y^0) - \nabla f_k(y^0) + \nabla f_k(y^s) - v^s}^2} \le \sum_{j=1}^s \ex{\norm{v^j-v^{j-1}}^2}  \le  \frac{\eta L}{2-\eta L}\ex{\norm{v^0}^2}.
\end{align}
By Assumption \ref{assum:f_restricted}, we have
\begin{align}\label{eq:term2}
	&\quad \norm{\nabla f(y^0) - \nabla f_k(y^0) + \nabla f_k(y^s)   - \nabla f(y^s)  }^2\nonumber \\
	&= \norm{\nabla (f-f_k)(y^0) - \nabla (f-f_k)(y^s)}^2\nonumber\\
	&\le 2c^2\norm{y^s-x^*}^2+2c^2\norm{y^0-x^*}^2.
\end{align}
Therefore, combining \eqref{eq:term1} and \eqref{eq:term2}, we have
\begin{align*}
& \ex{\norm{\nabla f(y^s) - v^s}^2} \\
\le& 2 \ex{\norm{\nabla f(y^0) - \nabla f_k(y^0) + \nabla f_k(y^s)   - \nabla f(y^s)  }^2}+2\ex{\norm{\nabla f(y^0) - \nabla f_k(y^0) + \nabla f_k(y^s) - v^s}^2}\\
\le&  4c^2\ex{\norm{y^s-x^*}^2}+4c^2\ex{\norm{y^0-x^*}^2} + \frac{2\eta L}{2-\eta L}\ex{\norm{v^0}^2} .
\end{align*}
Substituting it into Lemma \ref{lem:grad_comp} gives 
\[
\begin{aligned}
	&\sum_{s=0}^{m-1} \ex{\norm{\nabla f(y^s)}^2} \\
		\le& \frac{2}{\eta} \ex{f(y^0) - f(x^*)} + \sum_{s=0}^{m-1} \ex{\norm{\nabla f(y^s) - v^s}^2} \\
	\le& \frac{2}{\eta} \ex{f(y^0 )- f(x^*)} +\sum_{s=0}^{m-1} \Big(4c^2\ex{\norm{y^s-x^*}^2}+4c^2\ex{\norm{y^0-x^*}^2} + \frac{2\eta L}{2-\eta L}\ex{\norm{v^0}^2} \Big) .
\end{aligned}
\]
Since $f$ is $\sigma$-strongly convex, we have $4c^2\norm{y^s-x^*}^2 \le \frac{4c^2}{\sigma^2}\norm{\nabla f(y^s)}^2$. Denote $y^+$ as the local update which is selected from $y^0, \cdots, y^{m-1}$ uniformly at random. We have
\[
\begin{aligned}	
	& \paren{1-\frac{4c^2}{\sigma^2}}\ex{\norm{\nabla f(y^+)}^2} \\
	=& \paren{1-\frac{4c^2}{\sigma^2}}\frac{1}{m}\sum_{s=0}^{m-1} \ex{\norm{\nabla f(y^s)}^2}  + \frac{2\eta L}{2-\eta L}\ex{\norm{v^0}^2} \\
	\le& \paren{\frac{1}{\sigma\eta m} + \frac{4c^2}{\sigma^2}+\frac{2\eta L}{2-\eta L}}\ex{\norm{\nabla f(y^0)}^2}.
\end{aligned}
\]
Since $\tx^{t+1}$ is randomly chosen from the local outputs $\{y_k^{t+},\ 1\le k \le n\}$, we have
\begin{align*}
& \paren{1-\frac{4c^2}{\sigma^2}}\ex{\norm{\nabla f(\tx^{t+1})}^2} \le\paren{\frac{1}{\sigma\eta m} + \frac{4c^2}{\sigma^2}+\frac{2\eta L}{2-\eta L}}\ex{\norm{\nabla f(\tx^t)}^2}.
\end{align*}
 \end{proof}

\subsection{Proof of Theorem~\ref{thm:SARAH_nonc}}\label{proof:thm:SARAH_nonc}
   
Recall Lemma \ref{lem:grad_comp}. The theorem follows if  
	$$\sum_{s=0}^{m-1} \ex{\norm{\nabla f(y^s) - v^s}^2} - (1-L\eta)\sum_{s=0}^{m-1}\ex{\norm{v^s}^2} \leq 0 . $$ 
The rest of this proof is thus dedicated to show the above inequality. Note that
    \begin{align*}
    &\ex{\norm{\nabla f(y^s) - v^s}^2} \\
    \le& 2 \ex{\norm{\nabla f(y^0) - \nabla f_k(y^0) + \nabla f_k(y^s) -  \nabla f(y^s) }^2}+2\ex{\norm{\nabla f(y^0) - \nabla f_k(y^0) + \nabla f_k(y^s) - v^s}^2}\\
    \le& 2c^2 \ex{\norm{y^0 - y^s}^2} + 2\sum_{j=1}^s \ex{\norm{v^j-v^{j-1}}^2}\\
    \le& 2c^2s\sum_{j=1}^s \ex{\norm{y^j-y^{j-1}}^2} + 2\sum_{j=1}^s \ex{\norm{v^j-v^{j-1}}^2}\\
    =& 2c^2s\eta^2 \sum_{j=1}^s \ex{\norm{v^{j-1}}^2} + 2\sum_{j=1}^s \ex{\norm{v^j-v^{j-1}}^2},
    \end{align*}
    where the second inequality follows from Lemma~\ref{lem:trace_concaten} and Assumption~\ref{assum:f_nrestricted}, and the third inequality follows from $y^0-y^s = \sum_{j=1}^s y^j -y^{j-1}$, and the last line follows from the definition. The $L$-smoothness of $ \ell_z$ implies that
    \begin{align*}
    	\norm{v^j-v^{j-1}}^2 &= \norm{\nabla \ell_z\paren{y^j}-\nabla \ell_z\paren{y^{j-1}}}^2 \le L^2\norm{y^j-j^{j-1}}^2 = L^2\eta^2 \norm{v^{j-1}}^2.
    \end{align*}
    So we have
    \begin{align*}
	    &\sum_{s=0}^{m-1} \ex{\norm{\nabla f(y^s) - v^s}^2} - (1-L\eta)\sum_{s=0}^{m-1}\ex{\norm{v^s}^2} \\
	    \le&\sum_{s=1}^{m-1} (2c^2s+2L^2)\eta^2 \sum_{j=1}^s \ex{\norm{v^{j-1}}^2} - (1-L\eta)\sum_{s=0}^{m-1}\ex{\norm{v^s}^2}\\
	    \le&\sum_{s=0}^{m-1}\paren{m(m-1)c^2\eta^2 + 2L^2\eta^2(m-1)-(1-L\eta)}\cdot\ex{\norm{v^s}^2}.
    \end{align*}
    Therefore, with $0 < \eta \le \frac{2}{L\paren{1+\sqrt{1+8(m-1)+4m(m-1)c^2/L^2}}}$, we have $m(m-1)c^2\eta^2 + 2L^2\eta^2(m-1)-(1-L\eta) \le 0$ and the proof is finished.

\subsection{Proof of Lemma \ref{lem:trace_concaten}}
\label{proof:lem:trace_concaten}

First, we write 
\begin{align*}
& \nabla f(y^0) - \nabla f_k(y^0) + \nabla f_k(y^s) - v^s \\
=&\nabla f(y^0) - \nabla f_k(y^0) + \nabla f_k(y^{s-1}) - v^{s-1}+ [\nabla f_k(y^s) - \nabla f_k(y^{s-1})] - [v^s - v^{s-1}].
\end{align*}
Let $\mathscr{F}_s$ denote the $\sigma$-algebra generated by all random sample selections in sub-iteration $0, \cdots, s-1$. We have
\begin{align*}
	&\ex{\norm{\nabla f(y^0) - \nabla f_k(y^0) + \nabla f_k(y^s) - v^s}^2 \cond \mathscr{F}_s} \\
	=& \norm{\nabla f(y^0) - \nabla f_k(y^0) + \nabla f_k(y^{s-1}) - v^{s-1}}^2 +  \norm{\nabla f_k(y^s) - \nabla f_k(y^{s-1})}^2   + \ex{\norm{v^s - v^{s-1}}^2 \cond \mathscr{F}_s}  \\
	& + 2\big\langle \nabla f(y^0) - \nabla f_k(y^0) + \nabla f_k(y^{s-1}) - v^{s-1},  \nabla f_k(y^s) - \nabla f_k(y^{s-1}) \big\rangle  \\
	& -2\innprod{ \nabla f(y^0) - \nabla f_k(y^0) + \nabla f_k(y^{s-1}) - v^{s-1}, \ex{v^s - v^{s-1} \cond \mathscr{F}_s}}  \\
	&-2 \innprod{\nabla f_k(y^s) - \nabla f_k(y^{s-1}), \ex{v^s - v^{s-1} \cond \mathscr{F}_s}} \\
	 =& \norm{\nabla f(y^0) - \nabla f_k(y^0) + \nabla f_k(y^{s-1}) - v^{s-1}}^2  - \norm{\nabla f_k(y^s) - \nabla f_k(y^{s-1})}^2 + \ex{\norm{v^s - v^{s-1}}^2\cond \mathscr{F}_s},
\end{align*}
where the second equality follows from
\begin{align*}
	\ex{v^s - v^{s-1} \cond \mathscr{F}_s} & = \ex{\nabla \ell_z(y^s) - \nabla \ell_z(y^{s-1}) \cond \mathscr{F}_s}  = \nabla f_k (y^s) - \nabla f_k (y^{s-1}),
\end{align*}

Taking expectation over $\mathscr{F}_s$ gives
\begin{align*}
		&\ex{\norm{\nabla f(y^0) - \nabla f_k(y^0) + \nabla f_k(y^s) - v^s}^2} \\
		=&\ex{\norm{\nabla f(y^0) - \nabla f_k(y^0) + \nabla f_k(y^{s-1}) - v^{s-1}}^2}  - \ex{\norm{\nabla f_k(y^s) - \nabla f_k(y^{s-1})}^2} + \ex{\norm{v^s - v^{s-1}}^2}.
	\end{align*}
Hence telescoping the above equality we obtain the claimed result.

\section{Proof for D-MiG (Theorem~\ref{thm:dismig})} \label{sec:proof_dmig}
As earlier, we simplify the notations $y_k^{t,s}$, $x_k^{t,s}$ and $v_k^{t,s}$ by dropping the  superscript $t$ and the  subscript $k$. In this section we deal with the non-smooth target function $F(x) = f(x) + g(x)$ as mentioned in Remark~\ref{remark:dig}, where $g$ is a convex and non-smooth function known to all agents. The analysis is done by carefully adapting the proof for the centralized algorithm (i.e. \cite[Section B.1.]{zhou2018simple}) with the distributed smoothness assumption. 
We impose the following constraint on the step size $\eta$:
\begin{equation}\label{eq:stepsize_constraint}
	L\theta + \frac{L\theta}{1-\theta} + c \le \frac{1}{\eta}.
\end{equation}
We restate the inequalities (8) and (9) \cite{zhou2018simple} and change notations to match our context:
\begin{align}
	f(y^{s-1}) - f(u) &\le \frac{1-\theta}{\theta}\innprod{\nabla f(y^{s-1}), \tx - y^{s-1}} + \innprod{\nabla f(y^{s-1}), x^{s-1}-u}, 	\label{eq:MiG_1} \\
	\innprod{\nabla f(y^{s-1}), x^{s-1}-u}&  = \innprod{\nabla f(y^{s-1}) - \tilde{\nabla}, x^{s-1}-u} + \innprod{\tilde{\nabla}, x^{s-1}-x^s} + \innprod{\tilde{\nabla}, x^{s}-u},
	\label{eq:MiG_2}
\end{align}
where $\tilde{\nabla} = \nabla\ell_z(y^{s-1})-\nabla\ell_z(\tx)+\nabla f(\tx)$, with $z$ randomly selected from $\mathcal{M}_k$ and $u\in \mbR^d$ is an arbitrary vector. Following the $L$-smoothness argument in \cite{zhou2018simple}, \eqref{eq:MiG_2} leads to
\begin{equation*}
	\innprod{\tilde{\nabla}, x^{s-1}-x^s} \le \frac{1}{\theta}\paren{f(y^{s-1})-f(y^s)}+\innprod{\nabla f(y^{s-1})-\tilde{\nabla}, x^{s}-x^{s-1}} + \frac{L\theta}{2}\norm{x^{s}-x^{s-1}}^2 .
\end{equation*}
By plugging in the constraint \eqref{eq:stepsize_constraint}, we have
\begin{align}
\innprod{\tilde{\nabla}, x^{s-1}-x^s} \le &\frac{1}{\theta}\paren{f(y^{s-1})-f(y^s)}+\innprod{\nabla f(y^{s-1})-\tilde{\nabla}, x^{s}-x^{s-1}} \notag \\
&+ \paren{\frac{1}{2\eta}-\frac{L\theta}{2(1-\theta)}-\frac{c}{2}}\norm{x^{s}-x^{s-1}}^2.
\label{eq:MiG_3}
\end{align}
By combining \eqref{eq:MiG_1},\eqref{eq:MiG_2},\eqref{eq:MiG_3}, \cite[Lemma 3]{zhou2018simple} and then taking expectation over the choice of random sample $z$, we have
\begin{equation}
	\begin{aligned}
		f(y^{s-1}) - f(u) \le& \frac{1-\theta}{\theta}\innprod{\nabla f(y^{s-1}), \tx - y^{s-1}} + \ex{\innprod{\nabla f(y^{s-1}) - \tilde{\nabla},x^{s} - u}} \\
		& + \frac{1}{\theta}\paren{f(y^{s-1}) - \ex{f(y^s)}} -\paren{\frac{L\theta}{2(1-\theta)}+\frac{c}{2}}\ex{\norm{x^s - x^{s-1}}^2} \\
		&+ \frac{1}{2\eta} \norm{x^{s-1}-u}^2 - \frac{1+\eta \sigma}{2\eta}\ex{\norm{x^s - u}^2} + g(u) - \ex{g(x^s)}.
	\end{aligned}
	\label{equ:mig_ref}
\end{equation}
We further split the term $\ex{\innprod{\nabla f(y^{s-1}) - \tilde{\nabla},x^{s} - u}}$ as
\begin{align*}
& \ex{\innprod{\nabla f(y^{s-1}) - \tilde{\nabla},x^{s} - u}} \\
=&\ex{\innprod{\nabla f(\tx) - \nabla f_k(\tx) +\nabla f_k(y^{s-1}) - \tilde{\nabla},x^{s} - x^{s-1}}} +\ex{\innprod{\nabla f(y^{s-1}) - \tilde{\nabla},x^{s-1} - u}}\\
&+\ex{\innprod{\nabla f(y^{s-1})-\nabla f(\tx) + \nabla f_k(\tx) -\nabla f_k(y^{s-1}),x^{s} - x^{s-1}}} \\
\le&\frac{1}{2\beta}\ex{\norm{\nabla f(\tx) - \nabla f_k(\tx) +\nabla f_k(y^{s-1}) - \tilde{\nabla}}^2} + \frac{\beta}{2}\ex{\norm{x^{s} - x^{s-1}}^2}\\
&+c\norm{\tx - y^{s-1}}\ex{\norm{x^s - x^{s-1}}}+ c\norm{\tilde{x}-y^{s-1}}\norm{x^{s-1}-u}\\
\le&\frac{1}{2\beta}\paren{2LD_f(\tx, y^{s-1})+2cL\norm{\tx-y^{s-1}}^2} + \frac{\beta + c}{2}\ex{\norm{x^{s} - x^{s-1}}^2} +c\norm{\tx - y^{s-1}}^2  + \frac{c}{2}\norm{x^{s-1}-u}^2,
\end{align*}
where the first inequality is due to Cauchy-Schwarz inequality and Assumption \ref{assum:f_nrestricted}, with $\beta > 0$ satisfying $\frac{1-\theta}{\theta} = \frac{L}{\beta}$, and the last inequality is obtained by combining
\[
\ex{\norm{\nabla f(\tx) - \nabla f_k(\tx) +\nabla f_k(y) - \tilde{\nabla}}^2} \le \ex{\norm{\nabla \ell_z(y^{s-1})-\nabla \ell_z(\tx)}^2}.
\]
with Lemma~\ref{lem:rsp666} under Assumption~\ref{assum:f_nrestricted}, along with the inequality $ab \le (a^2+b^2)/2$. By substituting the inequality into (\ref{equ:mig_ref}) we have the following result. 
\begin{align*}
  f(y^{s-1}) - f(u) \le& \frac{1}{2\beta}\ex{2L(f(\tx)-f( y^{s-1})+2cL\norm{\tx-y^{s-1}}^2} +c\norm{\tx - y^{s-1}}^2 + \frac{c}{2}\norm{x^{s-1}-u}^2\\
&+ \frac{1}{\theta}\paren{f(y^{s-1}) - \ex{f(y^s)}} + \frac{1}{2\eta} \norm{x^{s-1}-u}^2 - \frac{1+\eta \sigma}{2\eta}\ex{\norm{x^s - u}^2} + g(u) - \ex{g(x^s)}\\
\le& c\theta\norm{\tx-x^{s-1}}^2 + \frac{c}{2}\norm{x^{s-1}-u}^2+ \frac{1-\theta}{\theta}\paren{f(\tx)-f( y^{s-1})} \\
&+ \frac{1}{\theta}\paren{f(y^{s-1}) - \ex{f(y^s)}} + \frac{1}{2\eta} \norm{x^{s-1}-u}^2 - \frac{1+\eta \sigma}{2\eta}\ex{\norm{x^s - u}^2} + g(u) - \ex{g(x^s)}.
\end{align*}
Note that the choice of $\beta$ cancels both $\frac{1-\theta}{\theta}\innprod{\nabla f(y^{s-1}), \tx - y^{s-1}}$ and $\paren{\frac{L\theta}{2(1-\theta)}+\frac{c}{2}}\ex{\norm{x^s - x^{s-1}}^2}$. By rearranging the above inequality, we can cancel the term $f(y^{s-1})$. We further use $-g(x^s)\le\frac{1-\theta}{\theta}g(\tx) - \frac{1}{\theta}g(y^s)$, which leads to 
\[
\begin{aligned}
\frac{1}{\theta}\paren{\ex{F(y^s)} - F(u)} \le &c\theta\norm{\tx-x^{s-1}}^2 + \frac{1-\theta}{\theta}\paren{F(\tx)-F(u)} \\
&+ \frac{1+c\eta}{2\eta} \norm{x^{s-1}-u}^2 - \frac{1+\eta \sigma}{2\eta}\ex{\norm{x^s - u}^2}.
\end{aligned}
\]
Note that $\norm{\tx - x^{s-1}}^2 \le 2\norm{\tx - x^*}^2 + 2\norm{x^{s-1}-x^*}^2$. By setting $u = x^*$
and using $\sigma$-strongly-convexity of $F$ we have
\[
\begin{aligned}
\frac{1}{\theta}\paren{\ex{F(y^s)} - F^*}\le& \paren{\frac{1-\theta}{\theta}+\frac{4c\theta}{\sigma}}\paren{F(\tx)-F^*} \\
&+ \frac{1+(1+4\theta)c\eta}{2\eta} \norm{x^{s-1}-x^*}^2 - \frac{1+\eta \sigma}{2\eta}\ex{\norm{x^s - x^*}^2}.
\end{aligned}
\]
To simplify the analysis we impose another constraint $\theta \le 1/2$:
\begin{equation}
\begin{aligned}
\frac{1}{\theta}\paren{\ex{F(y^s)} - F^*}\le& \paren{\frac{1-\theta}{\theta}+\frac{4c\theta}{\sigma}}\paren{F(\tx)-F^*} \\
& \quad + \frac{1+3c\eta}{2\eta} \norm{x^{s-1}-x^*}^2 - \frac{1+\eta \sigma}{2\eta}\ex{\norm{x^s - x^*}^2}.
\end{aligned}
\label{equ:mig_onestep}
\end{equation}
Let $w = \frac{1+\eta\sigma}{1+3c\eta}$. Multiplying (\ref{equ:mig_onestep}) by $w^{s-1}$ and then summing over $s$, we have
\[
\begin{aligned}
&\frac{1}{\theta}\sum_{s=0}^{m-1} w^s(\ex{F(y^{s+1})}-F^*) + \frac{w^m(1+3c\eta)}{2\eta}\ex{\norm{x^m - x^*}^2} \\
\le& \paren{\frac{1-\theta}{\theta}+\frac{4c\theta}{\sigma}}\sum_{s=0}^{m-1} w^s(F(\tx)-F^*)+\frac{1+3c\eta}{2\eta}\norm{x^0 - x^*}^2,\\
\end{aligned}
\]
Adding the superscript $t$ and the subscript $k$ back and applying Jensen's inequality to the definition of $y_k^{t+}$, we get
\[
\begin{aligned}
&\frac{1}{\theta}\sum_{s=0}^{m-1} w^s(\ex{F(y_k^{t+})}-F^*) + \frac{w^m(1+3c\eta)}{2\eta}\ex{\norm{x^{t+1,0}_k - x^*}^2} \\
\le& \paren{\frac{1-\theta}{\theta}+\frac{4c\theta}{\sigma}}\sum_{s=0}^{m-1} w^s(F(\tx^t)-F^*)+\frac{1+3c\eta}{2\eta}\norm{x^{t,0}_k - x^*}^2.
\end{aligned}
\]
Averaging the inequality over $k = 1,\cdots, n$, we have
\[
\begin{aligned}
&\frac{1}{\theta}\sum_{s=0}^{m-1} w^s(\ex{F(\tx^{t+1})}-F^*) + \frac{w^m(1+3c\eta)}{2\eta}\ex{\frac{1}{n}\sum_{k=1}^n\norm{x^{t+1,0}_k - x^*}^2} \\
\le& \paren{\frac{1-\theta}{\theta}+\frac{4c\theta}{\sigma}}\sum_{s=0}^{m-1} w^s(F(\tx^t)-F^*)+\frac{1+3c\eta}{2\eta}\frac{1}{n}\sum_{k=1}^n\norm{x^{t,0}_k - x^*}^2.
\end{aligned}
\]
\paragraph{Case I}
$\frac{m(\sigma-3c)}{L} \le 3/4$. By setting $\theta = \sqrt{\frac{m(\sigma-3c)}{3L}} \le 1/2$ and $\eta = \frac{1}{3L\theta+c}$ we satisfy the step size constraint \eqref{eq:stepsize_constraint}. We aim to show that
\[
\frac{1}{\theta} \ge \paren{\frac{1-\theta}{\theta}+\frac{4c\theta}{\sigma}}w^m.
\]
We have $w = \frac{1+\eta\sigma}{1+3c\eta} \le 1 + \eta \sigma - 3c\eta$. Note that $L\theta = \sqrt{\frac{Lm(\sigma-3c)}{3}} \ge \sqrt{\frac{\sigma(\sigma-3c)}{3}} >c$, so $\eta \le \frac{1}{4L\theta}$. Moreover, with $\sigma > 8c$ we have  
\[
\frac{\frac{m}{4L\theta}(\sigma-3c)}{\theta-\frac{4c}{\sigma}\theta^2} = \frac{3}{4(1-4c\theta/\sigma)} \le 1.
\]
Hence,
\begin{align*}
\paren{1-\theta+\frac{4c}{\sigma}\theta^2}\paren{1+\eta\sigma-3\eta c}^m & \le \paren{1-\theta+\frac{4c}{\sigma}\theta^2}\paren{1+\frac{1}{4L\theta}(\sigma - 3c)}^m \\
&\le \paren{1-\zeta}\paren{1+\zeta/m}^m,
\end{align*}
where $\zeta =\frac{m}{4L\theta}(\sigma - 3c)$. With $\zeta \le \theta - \frac{4c}{\sigma}\theta^2 \le 1/2$ we have $\paren{1-\zeta}\paren{1+\zeta/m}^m \le 1$. So we have
\[
\paren{\frac{1-\theta}{\theta}+\frac{4c\theta}{\sigma}}w^m \ge \frac{1}{\theta}\paren{1-\theta+\frac{4c}{\sigma}\theta^2}\paren{1+\eta\sigma-3\eta c}^m \ge \frac{1}{\theta}.
\]

Therefore we have
\[
\begin{aligned}
&w^m\paren{\paren{\frac{1-\theta}{\theta}+\frac{4c\theta}{\sigma}}\sum_{s=0}^{m-1} w^s(F(\tx^{t+1})-F^*) + \frac{1+3c\eta}{2\eta}\ex{\frac{1}{n}\sum_{k=1}^n\norm{x^{t+1,0}_k - x^*}^2}} \\
\le& \paren{\frac{1-\theta}{\theta}+\frac{4c\theta}{\sigma}}\sum_{s=0}^{m-1} w^s(F(\tx^t)-F^*)+\frac{1+3c\eta}{2\eta}\frac{1}{n}\sum_{k=1}^n\norm{x^{t,0}_k - x^*}^2.
\end{aligned}
\]
The convergence rate over $T$ rounds of communication is $w^{-Tm} = \paren{1+\Theta\paren{\frac{1}{\sqrt{3\kappa m}}}}^{-Tm}$, so the communication complexity is $T=\mathcal{O}\paren{\sqrt{\kappa/m}\log(1/\epsilon)}$. With the choice of $m = \Theta(N/n)$ the runtime complexity is $\mathcal{O}((N/n+m)T)=\mathcal{O}(\sqrt{\kappa N/n}\log(1/\epsilon))$.

\paragraph{Case II}
$\frac{m(\sigma-3c)}{L} > 3/4$. We set $\theta = 1/2$ and $\eta = \frac{1}{3L\theta+c}$.
We have $\eta\sigma = \frac{1}{3\kappa/2 + c/\sigma} \le 2/3$, thus 
\[
w = \frac{1+\eta\sigma}{1+3c\eta} \ge \frac{1+2/3+\eta\sigma-3c\eta}{1+2/3} = 1 + \frac{3}{5}\paren{\eta \sigma - 3c\eta}
\]
and $w^m \ge 1+\frac{3}{5}(\sigma-3c)\eta m	\ge 1+\frac{3m(\sigma-3c)}{5(3L/2+L/8)}>1+1/4$. So we have
\[
\begin{aligned}
&2\sum_{s=0}^{m-1} w^s(F(\tx^{t+1})-F^*) + \frac{5}{4}\cdot\frac{1+3c\eta}{2\eta}\ex{\frac{1}{n}\sum_{k=1}^n\norm{x^{t+1,0}_k - x^*}^2} \\
\le& \paren{1+\frac{1}{4}}\sum_{s=0}^{m-1} w^s(F(\tx^t)-F^*)+\frac{1+3c\eta}{2\eta}\frac{1}{n}\sum_{k=1}^n\norm{x^{t,0}_k - x^*}^2.
\end{aligned}
\]
Therefore,
\[
\begin{aligned}
&\frac{5}{4}\paren{\frac{5}{4}\sum_{s=0}^{m-1} w^s(F(\tx^{t+1})-F^*) + \frac{1+3c\eta}{2\eta}\ex{\frac{1}{n}\sum_{k=1}^n\norm{x^{t+1,0}_k - x^*}^2}} \\
\le& \frac{5}{4}\sum_{s=0}^{m-1} w^s(F(\tx^t)-F^*)+\frac{1+3c\eta}{2\eta}\frac{1}{n}\sum_{k=1}^n\norm{x^{t,0}_k - x^*}^2.
\end{aligned}
\]
This implies that $\mathcal{O}(\log(1/\epsilon))$ rounds of communication is sufficient to find an $\epsilon$-accurate solution, and that the runtime complexity is $\mathcal{O}(N/n\log(1/\epsilon))$.

\section{Discussions on Distributed Smoothness}
\label{sec:discussions}
		In this paper, we have established that distributed variance reduced methods admit simple convergence analysis under the distributed smoothness for all worker machines, as long as the parameter $c$ is smaller than a constant fraction of $\sigma$, the strong convexity parameter of the global loss function $f$. In this section, we will show that the distributed smoothness can be guaranteed for many practical loss functions as long as the local data size is sufficiently large and homogeneous across machines. This is as expected, since SVRG and SARAH rely heavily on exploiting data similarities to reduce the variance. 
		
	    Since the distributed smoothness only examines the gradient information of $f_k -f$, it can be applied to loss functions with non-smooth gradients, e.g. Huber loss. However, for simplicity of exposition, we limit our focus to the case when the sample loss $\ell_z(\cdot)$ is second-order differentiable and demonstrate the smoothness of $f-f_k$ via {\em uniform concentration} of Hessian matrices. 
	    
	    For simplicity, we consider the quadratic loss case, which allows us to compare with existing result for the DANE algorithm \cite{Shamir2013}, which is a communication-efficient approximate Newton-type algorithm. Assume $\ell(\cdot,z)$ is quadratic for all $z$.
	Recall the following result on the concentration of Hessian matrices from \cite{Shamir2013}.
	\begin{lemma}\cite{Shamir2013}
	    \label{lem:quadratic_concen}
	    If $0\preceq \nabla^2 \ell_z(x) \preceq L$ holds for all $z$, then with probability at least $1-\delta$ over the samples, for all $x$,
	    $$\max_{1\leq k\leq n} \norm{\nabla^2 f_k(x) - \nabla^2 f(x)} \le \sqrt{\frac{32L^2\log(d n/\delta)}{N/n}}.$$
	\end{lemma}
	Moreover, the iteration complexity of DANE is given by the theorem below.
	\begin{theorem}\cite{Shamir2013}
	    \label{thm:DANE}
	    If $0\preceq \nabla^2 \ell_z(x) \preceq L$ holds for all $z$ and $\sigma \preceq \nabla^2 f(x) \preceq L$, then with probability exceeding $1-\delta$, DANE needs
	    \begin{equation*}
	    \mathcal{O}\paren{\frac{\kappa^2}{N/n}\log\paren{\frac{dn}{\delta}}\log\paren{\frac{L\norm{x^0-x^*}^2}{\epsilon}}}
	    \end{equation*}
	    iterations to find an $\epsilon$-optimal solution.
	\end{theorem}

By Theorem~\ref{thm:DANE},	\cite{Shamir2013} claims that when the local data size of every machine is sufficiently large, namely $N/n = \Omega\paren{\kappa^2 \log(dn)}$, DANE can find a desired $\epsilon$-optimal with $\mathcal{O}(\log(1/\epsilon))$ iterations and thus communication-efficient. Note that at this local data size, according to Lemma \ref{lem:quadratic_concen}, it is sufficient to establish $c=\mathcal{O}(\sigma)$, which satisfies the convergence requirement of D-SVRG and D-SARAH. Consequently, the proposed D-SVRG and D-SARAH converges at the same iteration complexity as DANE, that is $\mathcal{O}(\log(1/\epsilon))$. Recall that DANE requires its local routines to be solved exactly. In contrast, our results formally justifies that SVRG and SARAH can be safely used as an {\em inexact} local solver for DANE without losing its performance guarantees, thus answering an open question in \cite{reddi2016aide}.


\end{document}